\documentclass[conference]{IEEEtran}
\IEEEoverridecommandlockouts
\usepackage{cite}
\usepackage{amsthm}
\usepackage{multirow}
\usepackage{booktabs}
\newtheorem{theorem}{Theorem}
\newtheorem{assumption}{Assumption}
\newtheorem{lemma}{Lemma}
\usepackage{amsmath,amssymb,amsfonts}
\usepackage{algorithmic}
\usepackage{subcaption}
\usepackage{multirow}
\usepackage{algorithm}
\usepackage{algorithmic}
\usepackage{graphicx}
\usepackage{textcomp}
\usepackage{ulem}
\usepackage{xcolor}
\def\BibTeX{{\rm B\kern-.05em{\sc i\kern-.025em b}\kern-.08em
    T\kern-.1667em\lower.7ex\hbox{E}\kern-.125emX}}
\begin{document}
\renewcommand{\algorithmicrequire}{\textbf{Input:}}
\renewcommand{\algorithmicensure}{\textbf{Output:}}
\title{Towards Fair Graph Neural Networks via Graph Counterfactual without Sensitive Attributes

\thanks{
Tianlong Gu and Xuguang Bao are the corresponding authors.

This work was supported by the National Natural Science Foundation of China (Grant No. U22A2099 and Grant No. 62336003).}
}

\author{\IEEEauthorblockN{1\textsuperscript{st} Xuemin Wang}
\IEEEauthorblockA{\textit{Guilin University of Electronic Technology} \\
\textit{Guangxi Key Laboratory of Trusted Software}\\
Guilin, China \\
xueminwangbetter@163.com}
\and
\IEEEauthorblockN{2\textsuperscript{nd} Tianlong Gu}
\IEEEauthorblockA{\textit{Jinan University} \\
\textit{Engineering Research Center}\\
\textit{of Trustworthy AI (Ministry of Education)}\\
Guangzhou, China \\
gutianlong@jnu.edu.cn}
\and
\IEEEauthorblockN{3\textsuperscript{rd} Xuguang bao}
\IEEEauthorblockA{\textit{Guilin University of Electronic Technology} \\
\textit{Guangxi Key Laboratory of Trusted Software}\\
Guilin, China  \\
baoxuguang@guet.edu.cn}
\and
\IEEEauthorblockN{4\textsuperscript{th} Liang Chang}
\IEEEauthorblockA{\textit{Guilin University of Electronic Technology} \\
\textit{Guangxi Key Laboratory of Trusted Software
}\\
Guilin, China \\
changl@guet.edu.cn}

}

\maketitle

\begin{abstract}

Graph-structured data is ubiquitous in today’s connected world, driving extensive research in graph analysis. Graph Neural Networks (GNNs) have shown great success in this field, leading to growing interest in developing fair GNNs for critical applications. However, most existing fair GNNs focus on statistical fairness notions, which may be insufficient when dealing with statistical anomalies. Hence, motivated by the causal theory, there has been growing attention to mitigating root causes of unfairness utilizing graph counterfactuals.  Unfortunately, existing methods for generating graph counterfactuals invariably require the sensitive attribute. Nevertheless, in many real-world applications, it is usually infeasible to obtain sensitive attributes due to privacy or legal issues, which challenge existing methods. In this paper, we propose a framework named Fairwos (improving \underline{Fair}ness \underline{w}ith\underline{o}ut \underline{s}ensitive attributes). In particular, we first propose a mechanism to generate pseudo-sensitive attributes to remedy the problem of missing sensitive attributes, and then design a strategy for finding graph counterfactuals from the real dataset. To train fair GNNs, we propose a method to ensure that the embeddings from the original data are consistent with those from the graph counterfactuals, and dynamically adjust the weight of each pseudo-sensitive attribute to balance its contribution to fairness and utility. Furthermore, we theoretically demonstrate that minimizing the relation between these pseudo-sensitive attributes and the prediction can enable the fairness of GNNs. Experimental results on six real-world datasets show that our approach outperforms state-of-the-art methods in balancing utility and fairness.

\end{abstract}


\section{Introduction}

Graph-structured data is prevalent across a wide range of real-world applications, such as database management systems \cite{Mhedhbi2022}, E-commerce \cite{Li2020}, regional ambulance demand forecasting \cite{Wang2021}, traffic prediction \cite{Yuan2021}, drug discovery \cite{Saifuddin2023}, and knowledge graph reasoning \cite{chen2024}. In recent years, graph neural networks (GNNs) have demonstrated significant advantages in numerous graph-based analytical tasks, including node classification \cite{Peng2024}, link prediction \cite{Chen2020}, graph classification \cite{Wang2024}, and graph clustering \cite{Tsitsulin2023}. The reason is that GNNs effectively capture both node attributes and graph topology by aggregating information from neighboring nodes using a message-passing mechanism \cite{ref10}. However, recent studies \cite{ref11, ref12, ref13} show that the predictions of GNNs could be biased towards some demographic groups defined by sensitive attributes, such as race and gender, resulting in sensitive biases \cite{ref14}. Moreover, the message-passing mechanism may further magnify the bias \cite{ref15}, which limits the adoption of GNNs in many high-stake scenarios \cite{ref16, ref17}.

\begin{figure}[t]
    \centering
    \includegraphics[width=\linewidth]{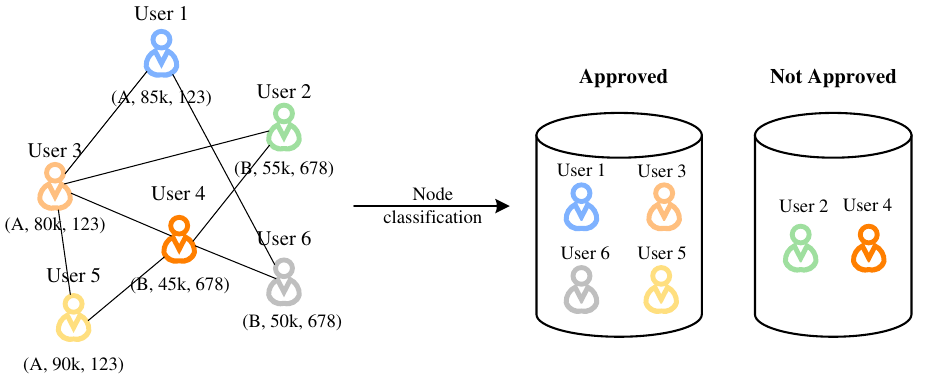}
    \caption{Running example of node classification in a loan approval scenario: users are classified based on their features (race, income, zip code) as well as their relationships with others to predict loan approval.}
    \label{fig:running_example}
\end{figure}

To address the issue of sensitive bias, many efforts have been made to introduce fairness consideration into GNNs, which achieves statistically fair predictions for various subgroups \cite{ref18}.  However, they fail to tackle sensitive bias in the presence of statistical anomalies \cite{ref15}. Hence, recent studies have turned to studying counterfactual fairness. In this context, a graph counterfactual \cite{ref19} is a statement regarding the hypothetical scenario based on the graph structure, generated by modifying certain elements of the original graph, such as node attributes or edge connections. In our running example (see Fig. 1), a graph counterfactual could be created by changing User 1's race from A to B while keeping other attributes constant. To implement counterfactual fairness, methods typically first generate graph counterfactuals by adjusting the sensitive attributes \cite{ref11, ref20}, and then ensure that both the original graph and its graph counterfactual yield similar predictions. However, recent research \cite{ref15} has pointed out that directly perturbing sensitive attributes can lead to non-realistic counterfactuals—the descriptions of hypothetical scenarios that contradict the real world. These counterfactuals often defy logic or common sense and are impossible or highly unlikely to occur in practice. In our running example (see Fig. 1), simply changing User 4's race without adjusting related attributes like income or postal code can create unrealistic counterfactuals, as these attributes are often correlated. Despite the above researches achieving superior performance, they assume the sensitive attributes available during the training process, which is difficult for many real-world applications due to various reasons, such as privacy and legal issues, or difficulties in data collection \cite{ref21, ref22, ref23}. For instance,  the General Data Protection Regulation (GDPR) \cite {ref24} has strict requirements for the process of collecting and using protected data.

Although sensitive attributes are unavailable during training, GNNs may still inherit sensitive biases \cite{ref14}. For example, in our running example (see Fig. 1), racial information may be excluded for privacy reasons, while features correlated with race, such as postal codes, remain in the training data, which introduces sensitive biases. Therefore, we introduce pseudo-sensitive attributes—proxy attributes derived from non-sensitive data within the graph.  These attributes can play a similar role as sensitive attributes and help the model detect and mitigate biases. For example, in our running example (see Fig. 1), postal code can serve as a pseudo-sensitive attribute that depicts the influence of race. To promote fairness using pseudo-sensitive attributes, we need to solve the following questions:

\begin{itemize}

\item \textbf{Pseudo-sensitive attribute generation}. Manually defining pseudo-sensitive attributes \cite{ref23,ref25} becomes challenging when dealing with high-dimensional features.

\item \textbf{Graph counterfactual generation}. Directly changing pseudo-sensitive attributes can lead to non-realistic counterfactuals that ignore the relationships between sensitive attributes and other features.\cite{ref15}

\item \textbf{Mitigating bias}. Existing research \cite{ref26} either improves worst-case subgroup performance while overlooking sensitive biases or focuses on fairness without addressing its root causes through graph counterfactuals \cite{ref27}.

\item \textbf{Balancing fairness and utility}. Different pseudo-sensitive attributes affect fairness and utility in various ways. Using the same regularization for all of them makes it difficult to balance fairness and utility.

\end{itemize}

To sum up, we address the above questions through the following contributions:

\begin{itemize}

\item We propose a framework Fairwos to promote fairness via graph counterfactuals without the sensitive attributes.

\item  We propose a mechanism to generate pseudo-sensitive attributes which are the representations of the graph structure and non-sensitive attributes.

\item We design a graph counterfactual searching strategy, which finds the graph counterfactual from a real data set and avoids the non-realistic counterfactual.

\item We propose a novel method to promote fairness by ensuring the embedding of individuals from original data and graph counterfactuals to be consistent, and design a weight updating mechanism to update the weight for pseudo-sensitive attributes based on their contribution to the utility and unfairness.

\item Extensive experiments on real-world datasets and state-of-the-art baselines show that Fairwos achieves a good balance between fairness and utility.

\end{itemize}

\section{Preliminaries}
	
	\subsection{Notations}
	Let $\mathcal{G}=(\mathcal{V}, \mathcal{E}, \mathbf{X})$ denote an attributed graph, where $\mathcal{V}=\{v_1,$  $\ldots, v_N\}$ is a set of nodes, $\mathcal{E} \subseteq \mathcal{V} \times \mathcal{V}$ is the set of edges, and $\mathbf{X}$ is the node attributes matrix. Here, $\mathbf{x}_{\mathbf{i}}=\left\{x_1\right.$, $\left.x_2, \ldots, x_n\right\}$ denotes the $i$-th row of $\mathbf{X}$, i.e., $\mathbf{x}_{\mathbf{i}}$ is the attribute vectors of the node $v_i$. The set of attributes is denoted as $F=$ $\left\{f_1, f_2, \ldots, f_n\right\}$. It is noted that sensitive attribute $S \notin F$. $\mathrm{A} \in$ $\mathbb{R}^{N \times N}$ denotes the adjacency matrix of graph $\mathcal{G}$, where $\mathbf{A}_{i j}=$ 1 if the edge between the node $v_i$ and node $v_i$, otherwise $\mathbf{A}_{i}=0$. In this paper, we focus on the semi-supervision setting where only a limited number of nodes $\mathcal{V}_L=$ $\left\{v_1, \ldots, v_l\right\} \subseteq \mathcal{V}$ are labeled for training and other nodes $\mathcal{V}_U=$ $\mathcal{V} \backslash \mathcal{V}_L$ are unlabeled.

\begin{figure*}[t]
		\centering
		\captionsetup{justification=centering}
		\includegraphics[width=\textwidth]{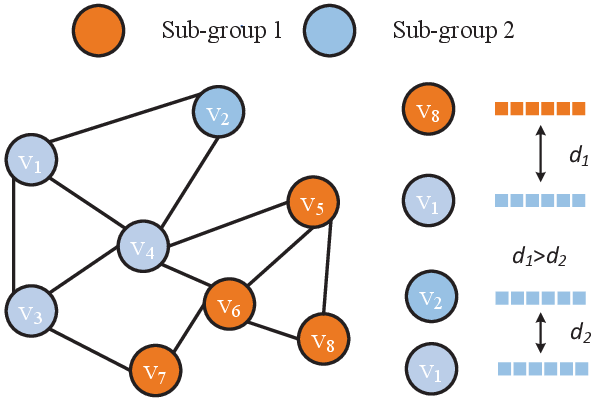} % 调整width参数以控制图片宽度
		\caption{An illustration of our proposed framework}
		\label{fig:example1}
	\end{figure*}

\subsection{Fairness Notations}

To verify the model's fairness, sensitive attributes can be requested during the testing phase for evaluating subgroup performance \cite{ref26,ref27}. To better understand and quantify this fairness, we introduce two formal fairness definitions applicable to binary labels $y \in \{0,1\}$ and a sensitive attribute $s \in \{0,1\}$. The classifier's prediction is denoted as $\hat{y} \in \{0,1\}$, where the classifier $\eta$ maps attribute vectors $\mathbf{x}$ to labels $y$.

\textbf{Definition 3.1. (Statistical Parity \cite{ref28}).} Statistical parity requires that the predictions be independent of the sensitive attribute $s$. Formally, this can be expressed as:
\begin{equation}
P(\hat{y} \mid s=0)=P(\hat{y} \mid s=1).
\end{equation}

\textbf{Definition 3.2. (Equal Opportunity \cite{ref29}).} Equal opportunity requires that the probability of a positive outcome for instances in the positive class should be the same for both subgroups. Formally, this is defined as:
\begin{equation}
P(\hat{y}=1 \mid y=1, s=0)=P(\hat{y}=1 \mid y=1, s=1).
\end{equation}

Statistical parity ensures the classifier maintains equal prediction across different subgroups, while equal opportunity expects the classifier to maintain equal true positive rates across different subgroups. We employ the following metrics to quantitatively assess statistical parity and equal opportunity:
\begin{equation}
\begin{gathered}
\Delta_{SP}=|P(\hat{y}=1 \mid s=0)-P(\hat{y}=1 \mid s=1)|, \\
\Delta_{EO}=|P(\hat{y}=1 \mid y=1, s=0)-P(\hat{y}=1 \mid y=1, s=1)|,
\end{gathered}
\end{equation}
where these probabilities are calculated on the test set. Based on these fairness notations, we define our problem as follows:

\subsection{Problem Definition}

  \textbf{Problem 1.} Given a graph $\mathcal{G}=(\mathcal{V}, \mathcal{E}, \mathbf{X})$ with a small label node set $\mathcal{V}_L \subseteq \mathcal{V}$, the attributes are denoted as $F=$ $\left\{f_1, f_2, \ldots, f_n\right\}$ and the sensitive attributes are not included in the attributes, i.e., $S \notin F$, our problem is to learn a fair GNN with high utility whilst guaranteeing the fairness criteria described above.

	\section{Methodology}

        \subsection{Overview}

	In this section, we propose a novel framework Fairwos, which aims to learn representation for promoting fairness without sensitive attributes.  As the illustration shown in Fig. 2, Fairwos mainly consists of five key components: (i) the encoder module reduces the dimension of input attributes; (ii) the GNN classifier leverages the processed attributes to learn the representation; (iii) the counterfactual data augmentation module finds graph counterfactuals for each factual observation with the guidance of the pseudo-sensitive attributes and labels; (iv) fair representation learning module encodes the graph counterfactual and the original data to gain the representations and minimize the disparity between these representations; (v) weight updating module adjusts the weight for each pseudo-sensitive attribute. The details of each component are introduced as follows:

 \subsection{Encoder Module}
        In previous studies, researchers have selected related non-sensitive attributes as pseudo-sensitive attributes based on prior knowledge \cite{ref23}. However, when dealing with high-dimensional attributes, it becomes challenging to determine these pseudo-sensitive attributes. At the same time, selecting appropriate pseudo-sensitive attributes is crucial for promoting fairness. To address this, we present the causal relationship between sensitive attributes and prediction in Fig. 3. Although sensitive attributes are excluded from the training process, they still influence the graph structure and non-sensitive attributes, thereby indirectly affecting the prediction. To capture this influence, we generate pseudo-sensitive attributes by incorporating representations of both the graph structure and non-sensitive attributes. As shown in Fig. 3, reducing the influence of each pseudo-sensitive attribute on the prediction helps minimize the impact of sensitive attributes. Therefore, we adopt an encoder to generate pseudo-sensitive attributes, as it can effectively learn complex relationships within high-dimensional data and capture the influence of both graph structure and non-sensitive attributes. The encoder is formulated as follows:
        
        \begin{equation}
		\widehat{\mathbf{Y}}=\operatorname{softmax}\left(\operatorname{Encoder}\left(\mathcal{G} ; \boldsymbol{\Theta}_{\text {enc }}\right) \cdot \mathbf{W}\right)
	\end{equation}
        where Encoder refers to a representation learning method with parameter set $\boldsymbol{\Theta}_{\mathrm{enc}}$, $\mathbf{W}$ is the weight matrix, $\mathcal{G}$ is the original graph data without sensitive attributes, and $\widehat{\mathbf{Y}}$ is the corresponding posterior class probabilities. To train this model, we minimize the cross-entropy (or any other classification-related) loss function $\ell(\cdot, \cdot)$ for the model parameters $\boldsymbol{\Theta}=\left\{\boldsymbol{\Theta}_{\text {enc }}, \mathbf{W}\right\}$ :
	\begin{equation}
		\boldsymbol{\Theta}^{\star}=\arg \min _{\boldsymbol{\Theta}} \sum_{v \in \mathcal{V}_L} \ell\left(\widehat{\mathbf{Y}}_v, \mathbf{Y}_v\right),
	\end{equation}
	where $\mathbf{Y}$ is the ground-truth labels, and $\mathcal{V}_L \subseteq \mathcal{V}$ is the set of training nodes. After pre-training, the encoder is used as a feature extractor for downstream tasks, providing low-dimensional node attributes, $\mathbf{X}^{(0)}$:
	\begin{equation}
		\mathbf{X}^{(0)}=\operatorname{Encoder}\left(\mathcal{G} ; \boldsymbol{\Theta}_{\text {enc }}^{\star}\right) .
	\end{equation}
    where $\boldsymbol{\Theta}_{\text {enc }}^{\star}$ refers to the optimal encoder parameter set. 
        
     \begin{figure}[t]
		\centering
		\includegraphics[width=0.5\textwidth]{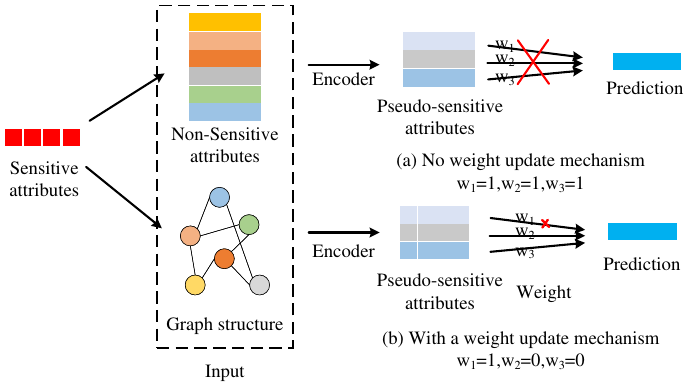} % 调整width参数以控制图片宽度
		\caption{ A causal relationship exists between sensitive attributes and predictions, where pseudo-sensitive attributes are low-dimensional representations of non-sensitive attributes and graph structure. (a) Applying regularization to each pseudo-sensitive attribute ensures that sensitive attributes do not influence predictions; (b) To balance the impact of pseudo-sensitive attributes on utility and fairness, a weight-update mechanism is used to adjust the strength of fairness enhancement.}
		\label{fig:example2}
	\end{figure}

     \subsection{GNN Classifier}

        The GNN classifier inputs $\mathcal{G}=(\mathcal{V}, \mathcal{E}, \mathbf{X}^{(0)})$ and outputs the prediction. Since our proposed Fairwos is flexible for various backbones such as GCN \cite{ref30} and GIN \cite{ref31}. Generally, the GNN classifier learns the representation after $l$-th layers of iterations. Specifically, each layer aggregates the representation of the node and the representation of the neighbors. The basic operation of GNN between $l$-th and $l$+1-th layer can be formulated as follows:

	\begin{equation}
		\mathbf{a}_v^{(l+1)}=\operatorname{AGGREGATE}^{(l)}\left(\left\{\mathbf{h}_u^{(l)}: u \in \mathcal{N}(v)\right\}\right)
	\end{equation}
	where $\mathcal{N}(v)$ denotes the node $v$'s neighborhoods.

	\begin{equation}
		\mathbf{h}_v^{(l+1)}=\operatorname{COMBINE}^{(l+1)}\left(\mathbf{h}_v^{(l)}, \mathbf{a}_v^{(l+1)}\right)
	\end{equation}
	where $\mathbf{h}_v^{(l+1)}$ is the representation vector of the node $v_i$ at the  $l$+1-th layer, and $\mathbf{a}_v^{(l+1)}$ denotes the representation of node $v_i$. In this paper, we consider the node classification task. The node representation is fed as input to a linear classification layer, formulated as follows:

	\begin{equation}
		\hat{y}_v=\sigma\left(\mathbf{h}_v \cdot \mathbf{w}\right)
	\end{equation}
	where $\mathbf{w}$ denotes the parameter of the linear classification layer. To enforce the prediction $\hat{y}_v$ to be closer to the ground truth ${y_v}$, a loss function is defined as follows:

	\begin{equation}
		\min _\theta \mathcal{L}_U=-\frac{1}{\left|v_L\right|} \sum_{v \in v_L}\left[y_v \log \hat{y}_v+\left(1-y_v\right) \log \left(1-\hat{y}_v\right)\right]
	\end{equation}
	where $\theta$ denotes the parameter of the GNN classifier.

    \subsection{Counterfactual Data Augmentation}

     We obtain the pseudo-sensitive attributes $\mathbf{X^{(0)}}$ from the encoder module. Next, we aim to perturb the pseudo-sensitive attributes to construct the graph counterfactual. However, the generated graph counterfactuals for several pseudo-sensitive attributes may introduce too many non-realistic counterfactuals. Enforcing the original representation to be consistent with representations of these non-realistic counterfactuals may disrupt the underlying latent semantic structure and decrease the model's utility. To address this issue, we generate the graph counterfactual from the actual data by searching for some nodes with the same label and different $x_{i}^{0} \in \mathbf{X^{(0)}}$. For the subgraph $\mathcal{G}_i$, the task of finding its graph counterfactual is formulated as follows:
	
	\begin{equation}
		\mathcal{G}_j=\underset{\mathcal{G}_j \in \mathbb{G}}{\arg \min }\left\{\mathcal{D}\left(\mathcal{G}_i, \mathcal{G}_j\right) \mid  y_i=  y_j, x_{i}^{0}\neq x_{j}^{0}\right\}
	\end{equation}
	where $\left.\mathbb{G}=\left\{\mathcal{G}_i \mid v_i \in \mathcal{V}\right)\right\}$ and $D(\cdot, \cdot)$ is a metric to measure the similarity between the subgraph $\mathcal{G}_i$ and subgraph $\mathcal{G}_j$, such as L2 distance. However, it is difficult to calculate the similarity between the subgraphs, while it is easy to calculate the similarity by leveraging the subgraph representation $\mathbf{h}_{\mathbf{i}}$. The task of finding the graph counterfactual is formulated as follows:
	
	\begin{equation}
		\mathbf{h}_{\mathbf{j}}=\underset{h_j \in \mathbb{H}}{\arg \min }\left\{\mathcal{D}\left(\mathbf{h}_{\mathbf{i}}, \mathbf{h}_{\mathbf{j}}\right)| y_i=  y_j, x_{i}^{0}\neq x_{j}^{0}\right\}
	\end{equation}
	where $\left.\mathbb{H}=\left\{\mathbf{h}_{\mathbf{i}} \mid v_i \in \mathcal{V}\right)\right\}$, and $\mathcal{D}(\cdot, \cdot)$ is a distance metric such as L2 distance. Since our setting is semi-supervised and the labels are limited, we pre-train a GNN classifier to provide pseudo-labels for searching the graph counterfactuals. For each node, only searching for one graph counterfactual for each node may result in mitigating bias inefficiently. Meanwhile, we also consider preserving high utility. Hence, we choose the top-$K$ graph counterfactuals which are most similar to the original graph data with the same prediction and different $x^{0}$. The counterfactual data augmentation returns the graph counterfactuals for the node $v_i$ denoted as $\left\{\mathcal{G}_i^k \mid k=1,2, \ldots, K\right\}$.

    \subsection{Fair Representation Learning}

    Following the design in \cite{ref11}, the causal relationship between the pseudo-sensitive attributes and the prediction can be achieved by minimizing the distance between the representations derived from the graph counterfactual and the original data. For the node $v$ with attribute $x_i^{0}$, the original subgraph is $\mathcal{G}_i$. Here, the counterfactual argumentation module returns the graph counterfactuals $\left\{\mathcal{G}_i^k \mid k=1,2, \ldots, K\right\}$. The GNN encoder is denoted as $f_{\mathcal{G}}$, the subgraph $\mathcal{G}_{i}$ inputs into the $f_{\mathcal{G}}$ and the representation is returned as $\mathbf{h_i}=f_{\mathcal{G}}\left(\mathcal{G}_i\right)$. Besides, the representation for graph counterfactuals are denoted as $\overline{\mathbf{h}}_{i}^{\mathbf{k}}=f_{\mathcal{G}}\left(\mathcal{G}_i^k\right)$ . To reduce the causal relationship, we minimize the disparity between the representations learned from $\mathbf{h}_i$ and $\overline{\mathbf{h}}_i^{\mathbf{k}}$. The loss function is formulated as follows:
	
	\begin{equation}
		\min _\theta \mathcal{D}=\sum_{k=1}^K \mathcal{D}_i \left(\mathbf{h}_i, \overline{\mathbf{h}}_{i}^{\mathbf{k}}\right)
	\end{equation}
	where $\mathcal{D} (\cdot, \cdot)$ denotes the distance metric such as L2 distance, the $K$ refers to the top-$K$ graph counterfactuals. The pseudo-sensitive attributes are not unique. Besides, these attributes contribute to bias variously and they affect the prediction differently. It is necessary to consider the importance of the pseudo-sensitive attributes. For these attributes, $x_{i}^{0} \in$ $\mathbf{X}^{0}$, where $1 \leq i \leq I$, we align the weight for each regularization term, the loss function is formulated as follows:
	
	\begin{equation}
		\min _\theta \mathcal{D}_{\text {all }}=\sum_{i=1}^I \lambda_i \cdot \mathcal{D}_i
	\end{equation}
	where $\lambda_i$ denotes the importance weight for the attribute $x_i \in \mathbf{x}^{0}$. A larger $\mathcal{D}_i$ indicates that the attribute $x_i$ has a strong causal relationship with the prediction. Hence, we would prefer to align a large $\lambda_i$ to more effectively reduce the causal relationship. Otherwise, the small $\lambda_i$ is preferred. The large $\lambda_i$ results in a weak relationship between the attribute $x^0_i$ and the prediction $\hat{y}_v$. Besides, it leads to little contribution of $x^0_i$ to the prediction. Since the learning process is dynamic, the predefined $\lambda_i$ leads to a failure to achieve a good trade-off between fairness and utility.  Hence, the $\lambda_i$ needs to be updated based on the regularization loss. The full optimization objective function is formulated as follows:
	
	\begin{equation}
		\begin{gathered}
			\min _{\theta, \lambda} \mathcal{L}_{U}+\alpha \cdot \sum_{i=1}^I \lambda_i \cdot \sum_{k=1}^K \mathcal{D}_i \left(\mathbf{h}_{\mathbf{i}}, \overline{\mathbf{h}}_{i}^{\mathbf{k}}\right)+\|\lambda\|_2^2 \\
			\text { s.t. } \lambda_i \geq 0, \forall x^{0}_i \in \mathbf{x}^{0}  ; \sum_{i=1}^I \lambda_i=1
		\end{gathered}
	\end{equation}
	where $\alpha$ denotes the weights of the regularization term.

	\subsection{Training Algorithm}
	
	It is a challenge to optimize the parameters of the GNN classifier and $\lambda$. Thus, we adopt an alternating optimization schema to update these parameters iteratively in this work. 
	
	\textbf{Update $\boldsymbol{\theta}$.} To update $\theta$, we fix $\lambda$ and remove other parameters that are irrelevant to $\theta$. The objective function in equation (15) reduces to equation (16), which is formulated as follows:
	
	\begin{equation}
		\min _\theta \mathcal{L}_{U}+\alpha \cdot \sum_{i=1}^I \lambda_i \cdot \sum_{k=1}^K \mathcal{D}_i\left(\mathbf{h}_{\mathbf{i}}, \overline{\mathbf{h}}_{i}^{\mathbf{k}}\right)+\|\lambda\|_2^2
	\end{equation}
	where $\theta$ is updated via stochastic gradient descent.
	
	\textbf{Update $\lambda_i$.} To update $\lambda_i$, the objective function in equation (15) is reduced to equation (17) as follows:
	
	\begin{equation}
		\begin{gathered}
			\min _\lambda \alpha \cdot \sum_{i=1}^I \lambda_i \cdot \sum_{k=1}^K \mathcal{D}_i \left(\mathbf{h}_{\mathbf{i}}, \overline{\mathbf{h}}_{i}^{\mathbf{k}}\right)+\|\lambda\|_2^2 \\
			\text { s.t. } \lambda_i \geq 0, \forall  x^{0}_i \in \mathbf{x}^{0} ; \sum_{i=1}^I \lambda_i=1
		\end{gathered}
	\end{equation}
	
	The equation (17) can be transformed as follows:
	
	\begin{equation}
		\begin{gathered}
			\min _\lambda \alpha \cdot \sum_{i=1}^I \lambda_i \cdot \sum_{k=1}^K \mathcal{D}_i\left(\mathbf{h}_{\mathbf{i}}, \overline{\mathbf{h}}_{i}^{\mathbf{k}}\right)+\|\lambda\|_2^2 \\
			\text { s.t. } -\lambda_i \leq 0, \forall  x^{0}_i \in \mathbf{x}^{0} ; \sum_{i=1}^I \lambda_i-1=0
		\end{gathered}
	\end{equation}
	
    Based on equation (18), we construct the Lagrange function as follows:

	\begin{equation}
		L(\lambda, a, b)=\alpha \cdot \sum_{i=1}^I \lambda_i \cdot \mathcal{D}_i^K +\|\lambda\|_2^2 - a\cdot\lambda+b \cdot \sum_{i=1}^I (\lambda_i-1)
	\end{equation}
	where $\sum_{k=1}^K \mathcal{D}_i\left(\mathbf{h}_{\mathbf{i}}, \overline{\mathbf{h}}_{i}^{\mathbf{k}}\right)$ is represented as $\mathcal{D}_i^K$. Here $a$ and $b$ are Lagrange multipliers. The KKT condition is calculated as follows:

	\begin{equation}
		\left\{\begin{array}{l}
			\mathcal{D}_i^K+ 2\lambda_i-a_{i}+b=0, \forall i ; \\
			a_i \cdot \lambda_i=0, \forall i ; \\
			\lambda_i \geq 0 \quad \forall i ; \quad \sum_{i=1}^I \lambda_i=1 ; \\
			a_i \geq 0 \quad \forall i .
		\end{array}\right.
	\end{equation}
	We can obtain the $\lambda_i$ as follows:
	
	\begin{equation}
		\lambda_i=\frac{a_i-b- \mathcal{D}_i^K}{2 }, \quad i=1, \ldots, K
	\end{equation}

	For the second condition $a_i \cdot \lambda_i=0, \forall i$, we can know that $\lambda_i=\max \left\{0, \frac{-b-\mathcal{D}_i^K}{2 }\right\}$. Hence, it is necessary to calculate $b$ first. Since $\sum_{i=1}^I \lambda_i=1$, we can obtain the following equation: 
	
	\begin{equation}
		\sum_{i=1}^I \max \left\{0,-b-\mathcal{D}_i^K\right\}=2
	\end{equation}

	To simplify the representation, we denote $\mathcal{D}_i^K$ as $\mathcal{D}_i$. To solve the equation (22), we first $\operatorname{rank} \mathcal{D}_i$ in descending order and the ranking list is formulated as $\{\mathcal{D}_{1}, \mathcal{D}_{2},..., \mathcal{D}_{I}\}$. Assume that $b$ is within $\left[-\mathcal{D}_{j-1}^{\prime},-\mathcal{D}_j^{\prime}\right]$, the above equation is reduced to

	\begin{equation}
		b=-\frac{2+\sum_{i=j}^I \mathcal{D}_i^{\prime}}{I-j+1}
	\end{equation}

	Hence, $\lambda_i$ can be calculated as follows:
	
	\begin{equation}
		\lambda_i=\max \left\{0, \frac{2+\sum_{i=j}^I \mathcal{D}_{i}^{{\prime}} -\mathcal{D}_{i} \cdot (I-j+1)} {2 \cdot (I-j+1) }\right\}
	\end{equation}

	Based on this updating method, we formulate the training algorithm of Fairwos in algorithm 1. Specifically, it first initializes the encoder in line 1 and $\lambda$ in line 2. Then it extracts the low-dimensional node attributes and denotes them as the pseudo-sensitive. The pseudo-sensitive attributes are leveraged to pre-train GNN for the node classification task from line 3 to line 4. Next, the model is fine-tuned to improve the fairness from line 5 to line 8. Specifically, it finds the graph counterfactual in the training set and obtains the representations from the graph counterfactual and the original graph from line 6 to line 7. The parameter $\theta$ is updated by minimizing the disparity between these representations. At this time, the $\lambda$ is fixed; in the next step, we learn the parameter $\lambda_i$ based on the updated $\theta$ using the KKT in line 9 to line 12. Finally, this algorithm returns the parameter of the GNN classifier.

	\begin{algorithm}[tb]
		\caption{Training Algorithm of Fairwos}
		\label{alg:algorithm}
		\begin{algorithmic}[1]
			\REQUIRE $\mathcal{G}=(\mathcal{V}, \mathcal{E}, \mathbf{X}), \mathbf{y} \in \mathbb{R}^C, \alpha$%[1] enables line numbers
			\ENSURE $\theta$
			\STATE Initialize the encoder by optimizing (2).
			\STATE Initialize $\lambda$ as $1 / n$ where $n$ denotes encoder's dimension.
			\STATE Extract the low-dimensional node attributes $\mathbf{X}^{(0)}$
			\STATE Initialize $\theta$ by optimizing Eq. (7)
			\REPEAT
			\STATE Obtain the graph counterfactual $\left\{\mathcal{G}_i^k \mid k=1,2, \ldots, K\right\}$ for each node $v_i$.
			\STATE Gain the representation for the graph counterfactuals and representation for the original graph.
			\STATE Update $\theta$ based on Eq. (13)
			\STATE Rank $\mathcal{D}_j$ in descending order $\{\mathcal{D}_{1}, \mathcal{D}_{2},..., \mathcal{D}_{K}\}$
			\IF {$-\frac{2+\sum_{i=j}^I \mathcal{D}_i^{\prime}}{I-j+1} \in \left[-\mathcal{D}_{j-1}^{\prime},-\mathcal{D}_j^{\prime}\right]$}
			\STATE Calculate $\lambda_i$ based on Eq. (24).
			\ENDIF
			\UNTIL convergence
			\STATE \textbf{return} $\theta$
		\end{algorithmic}
	\end{algorithm}

	\section{Theoretical Analysis of Fairwos}
        \subsection{Fairness Theoretical Analysis}
	In this subsection, we provide a detailed theoretical analysis of our framework, Fairwos. Here, we provide, more specifically, a theoretical upper bound on the unfairness of the resulting representations.  Lastly, we provide a reason why the fairness promotion can reduce the upper bound.
	
	\begin{theorem}
		(The upper bound of unfairness). Given the representation $\mathbf{\mathcal{G}_u}$, $\mathbf{z_u}$, $\mathbf{x^{0}_u}$ and $\mathbf{\hat{y}_{u}}$, where $\mathbf{\mathcal{G}_u}$ denotes the sub-graph of node $u$, $\mathbf{z_u}$ refers to the representation learned by GNN, $\mathbf{x^{0}_u}$ denotes the peseudo-sensitive attributes and $\mathbf{\hat{y}_{u}}$ is the prediction by the $\mathbf{z_u}$, the mutal information between $s$ and $z_{u}$ is formualted as follows: $0 \leq I\left(s ; \mathbf{\hat{y}_{u}}\right) \leq I\left(s ; \mathbf{z_u}\right) \leq I\left(\mathcal{G}_u ; \mathbf{z}_u\right) \leq \sum_t^T I\left(x^{0}_i ; \mathbf{z}_u\right)$. Here, $T$ refers to the pseudo-sensitive attributes that are highly related to the representation $z_u$.
	\end{theorem}
	
	\begin{proof}
		The downstream classifier uses our framework's representation $z_u$ to predict the label $y_u$ of node $u$, thus forming a Markov chain $s \rightarrow \mathcal{G} \rightarrow \mathbf{z_u} \rightarrow \mathbf{\hat{y}_{u}}$. Essentially, promoting fairness is to decrease the relation between the sensitive attribute and the representation $\mathbf{z}_u$, i.e., the mutual information between $s$ and $\mathbf{z}_u$. Due to the existence of an encoder, the Markov chain is represented as $s \rightarrow \mathcal{G}_u \rightarrow \mathbf{x^{0}_u} \rightarrow \mathbf{z_u} \rightarrow \mathbf{\hat{y}_{u}}$. 
  We show that the downstream classifiers that leverage the representations output by Fairwos satisfy fairness as well.  $I\left(s; \mathbf{z_u}\right)  \leq I\left(\mathcal{G}_u; \mathbf{z}_u\right)$ represent that the bias is caused by the non-sensitive information (i.e., sub-graph) when the sensitive attributes are absent. Furthermore $ I\left(\mathcal{G}_u; \mathbf{z}_u\right) \leq \sum_t^T I\left(x^{0}_i; \mathbf{z_u}\right)$ indicate that the bias is determined by $\mathbf{x^{0}_u}$
	\end{proof}
	
	We reduce the mutual information between $x^{0}_i$ and $z_u$, which reduces the relation between $s$ and $\hat{y}_{u}$. The mutual information is reduced by enforcing the embedding of original data to be consistent with the embedding of the searched graph counterfactual.  
	
	\begin{theorem}
		(The upper bound of the difference between the original data and the graph counterfactual) Given the representation $z_u$ and the graph counterfactual representation $\tilde{z}_u$, the distances between them for $T$ pseudo-sensitive attributes are formulated as follows:
	\end{theorem}

	\begin{equation}
		\sum_{t=i}^T\left\|\tilde{\mathbf{z}}_u^t-\mathbf{z}_u^t\right\|_p  \leq \sum_{t=i}^T \prod_{k=1}^K \left\|\mathbf{W}_a^k\right\|_p
	\end{equation}
	where $\mathbf{W}_a^k$ denotes the weight matrices of the $k-th$ GNN layer.

	\begin{proof}
		
		The representation output by layer $k$ of GNN is given by:

		\begin{equation}
			{\mathbf{z}}_u^k=\sigma\left(\mathbf{W}_a^k {\mathbf{z}}_u^{k-1}+\mathbf{W}_n^k \sum_{v \in \mathcal{N}(u)} \mathbf{z}_v^{k-1}\right)
		\end{equation}
		where $\mathcal{N}(u)$ is the neighborhood of node $u$. To simplify the representation, we only show the graph counterfactual for one pseudo-sensitive attribute. The difference between the node embeddings obtained after the message-passing in layer $k$ is:
		
		\begin{equation}
			\begin{aligned}
				\tilde{\mathbf{z}}_u^k - \mathbf{z}_u^k = 
				& \sigma\left( \mathbf{W}_a^k \tilde{\mathbf{z}}_u^{k-1} 
				+ \mathbf{W}_n^k \sum_{v \in \mathcal{N}(\tilde{u})} \mathbf{z}_v^{k-1} \right) \\
				& - \sigma\left( \mathbf{W}_a^k \mathbf{z}_u^{k-1} 
				+ \mathbf{W}_n^k \sum_{v \in \mathcal{N}(u)} \mathbf{z}_v^{k-1} \right)
			\end{aligned}
		\end{equation}
		
		Based on the inequality $\|\sigma(b)-\sigma(a)\|_p \leq\|b-a\|_p$, the euqation can be trainsformed as follows:

		\begin{equation}
			\begin{aligned}
				\left\|\tilde{\mathbf{z}}_u^k - \mathbf{z}_u^k\right\|_p 
				& \leq \left\| \mathbf{W}_a^k \left( \tilde{\mathbf{z}}_u^{k-1} - \mathbf{z}_u^{k-1} \right) \right. \\
				& \quad \left. + \mathbf{W}_n^k \left( \sum_{v \in \mathcal{N}(\tilde{u})} \mathbf{z}_v^{k-1} 
				- \sum_{v \in \mathcal{N}(u)} \mathbf{z}_v^{k-1} \right) \right\|_p
			\end{aligned}
		\end{equation}

		The second term in the above inequality is 0. Hence, we drop the second term and leverage Cauchy Schwartz inequality to get:
		
		\begin{equation}
			\begin{aligned}
				\left\|\tilde{\mathbf{z}}_u^k - \mathbf{z}_u^k\right\|_p 
				& \leq \left\| \mathbf{W}_a^k\left(\tilde{\mathbf{z}}_u^{k-1} - \mathbf{z}_u^{k-1}\right) \right\|_p \\
				& \leq \left\|\mathbf{W}_a^k\right\|_p \left\|\tilde{\mathbf{x}}^0_u - \mathbf{x}^0_u\right\|_p
			\end{aligned}
		\end{equation}
		
		Since $\tilde{\mathbf{x}}^{0}$ is generated by perturbing ${\mathbf{x}}^{0}$ with only a attribute. Hence, $\|\tilde{\mathbf{x}}^{0}_u-\mathbf{x}^{0}\|=1$ and the equation (29) is transformed as :
		\begin{equation}
			\left\|\tilde{\mathbf{z}}_u^k - \mathbf{z}_u^k\right\|_p \leq \left\|\mathbf{W}_a^k\right\|_p
		\end{equation}

		We consider that there exist $T$ pseudo-sensitive attributes and the layer of GNN is set as $K$. Hence, the difference in the upper bound is formulated as follows:
		
		\begin{equation}
			\sum_{t=i}^T\left\|\tilde{\mathbf{z}}_u^t-\mathbf{z}_u^t\right\|_p  \leq \sum_{t=i}^T \prod_{k=1}^K \left\|\mathbf{W}_a^k\right\|_p
		\end{equation}

	\end{proof}

\subsection{Convergence Analysis}

In this subsection, we analyze the convergence of Equation (16). Since the term $\|\lambda\|_2^2$ is independent of the parameters $\theta$, Equation (16) can be reformulated as:

\begin{equation}
\min_\theta \mathcal{L}(\theta) = \mathcal{L}_U(\theta) + \alpha \sum_{i=1}^I \lambda_i \sum_{k=1}^K \mathcal{D}_i(\mathbf{h}_i, \overline{\mathbf{h}}_i^k),
\end{equation}
where the regularization term $\mathcal{D}_i\left(\mathbf{h}_i, \overline{\mathbf{h}}_i^k\right)$ is defined as:

\begin{equation}
\mathcal{D}_i(\mathbf{h}_i, \overline{\mathbf{h}}_i^k) = \|\mathbf{h}_i - \overline{\mathbf{h}}_i^k\|_2^2.
\end{equation}

\begin{assumption}\label{assumption:bounded_derivatives}
The loss function $\mathcal{L}(\theta)$ has bounded first and second derivatives with respect to the parameters $\theta$. Specifically, there exist constants $G$ and $H$ such that for all $\theta$:
\begin{equation}
\left\|\nabla_\theta \mathcal{L}(\theta)\right\| \leq G \quad \text{and} \quad \left\|\nabla_\theta^2 \mathcal{L}(\theta)\right\| \leq H.
\end{equation}
\end{assumption}

\begin{lemma}\label{lemma:lipschitz_smooth}
Under Assumption \ref{assumption:bounded_derivatives}, the loss function $\mathcal{L}(\theta)$ is $L$-Lipschitz smooth, where $L = H$. Formally, for all $\theta_1, \theta_2$:
\begin{equation}
\left\| \nabla \mathcal{L}(\theta_1) - \nabla \mathcal{L}(\theta_2) \right\| \leq L \left\| \theta_1 - \theta_2 \right\|.
\end{equation}
\end{lemma}

\begin{theorem}\label{theorem:convergence_upper_bound}
{(Upper bound of the gradient in convergence)} Consider the loss function $\mathcal{L}(\theta)$ and let $\{\theta^t\}_{t=0}^T$ be the sequence of parameters generated by the gradient descent algorithm:
\begin{equation}
\theta^{k+1} = \theta^k - \alpha \nabla \mathcal{L}(\theta^k),
\end{equation}
where the learning rate $\alpha$ satisfies $0 < \alpha < \frac{2}{L}$. Under Assumption \ref{assumption:bounded_derivatives}, the following bound holds:
\begin{equation}
\min_{0 \leq k < T} \|\nabla \mathcal{L}(\theta^k)\|_2^2 \leq \frac{\mathcal{L}(\theta^0) - \mathcal{L}(\theta^*)}{M T},
\end{equation}
where $\mathcal{L}(\theta^*)$ is the minimal loss value achieved, and $M = \alpha - \frac{L \alpha^2}{2} > 0$.
\end{theorem}

\begin{proof}
To establish the convergence of the loss function $\mathcal{L}(\theta)$, we utilize the properties of Lipschitz smoothness and bounded gradients as specified in Assumption \ref{assumption:bounded_derivatives}.

Using Lemma \ref{lemma:lipschitz_smooth}, for a Lipschitz smooth function, we have:

\begin{equation}
\begin{aligned}
\mathcal{L}\left(\theta^{k+1}\right) \leq \mathcal{L}\left(\theta^k\right) 
& + \left\langle\nabla \mathcal{L}\left(\theta^k\right), \theta^{k+1}-\theta^k\right\rangle \\
& + \frac{L}{2}\left\|\theta^{k+1}-\theta^k\right\|^2
\end{aligned}
\end{equation}

To substitute the update rule $\theta^{k+1} = \theta^k - \alpha \nabla \mathcal{L}(\theta^k)$ into Equation (38), we obtain the following expression:

\begin{equation}
\mathcal{L}\left(\theta^{k+1}\right) \leq \mathcal{L}\left(\theta^k\right)-\alpha\left\|\nabla \mathcal{L}\left(\theta^k\right)\right\|^2+ \frac{L \alpha^2}{2}\left\|\nabla \mathcal{L}\left(\theta^k\right)\right\|^2
\end{equation}

Let $M = \alpha - \frac{L \alpha^2}{2} > 0$, ensuring that $0 < \alpha < \frac{2}{L}$, the equation (39) is transformed as:

\begin{equation}
\mathcal{L}\left(\theta^{k+1}\right) \leq \mathcal{L}\left(\theta^k\right)-M\left\|\nabla \mathcal{L}\left(\theta^k\right)\right\|^2
\end{equation}

We sum the inequality from $k = 0$ to $k = T-1$ and then take the average over $T$ iterations:

\begin{equation}
\frac{1}{T} \sum_{k=0}^{T-1}\left\|\nabla \mathcal{L}\left(\theta^k\right)\right\|^2 \leq \frac{\mathcal{L}\left(\theta^0\right)-\mathcal{L}\left(\theta^*\right)}{M T}
\end{equation}

By the pigeonhole principle, there exists at least one iteration $m \in \{0, \ldots, T-1\}$ such that:

\begin{equation}
\left\|\nabla \mathcal{L}\left(\theta^m\right)\right\|^2 \leq \frac{\mathcal{L}\left(\theta^0\right)-\mathcal{L}\left(\theta^*\right)}{M T}
\end{equation}

\end{proof}

As the number of iterations $T$ increases, the bound $\frac{\mathcal{L}(\theta^0) - \mathcal{L}(\theta^*)}{M T}$ approaches zero. This implies that the gradient norm gets smaller as $T$ grows. Therefore, the optimization algorithm converges to a stationary point.

	\begin{table*}
		\caption{Real-world dataset statistics}
             \centering
             \resizebox{\textwidth}{!}{
		\begin{tabular}{ccccccccc}
			\hline Dataset & \#Nodes & \#attributes & \#Edges & Average Degree & Sens. & Label & \#Train/Val/Test & Description \\
			\hline Bail & 18,876 & 18 & 311,870 & 34.04 & Race & Bail/no bail & $50 \% 25 \% 25 \%$ & Semi-synthetic \\
			Credit & 30,000 & 13 & 1,421,858 & 95.79& Age& default/no default  & $50 \% 25 \% 25 \%$ & Semi-synthetic\\
			Pokec-z & 67,797 & 277 & 617,958 & 19.23 & Region & Working Field & $50 \% 25 \% 25 \%$ & Facebook \\
			Pokec-n & 66,569 & 266 & 517,047 & 16.53 & Region & Working Field & $50 \% 25 \% 25 \%$ & Facebook \\
			NBA & 403 & 39 & 10,621 & 53.71 & Nationality & Salary & $50 \% 25 \% 25 \%$ & Twitter \\
			Occupation & 6,951 & 768 & 44,166 & 13.71 & Gender & Psy/CS & $50 \% 25 \% 25 \%$ & Twitter \\
			\hline
		\end{tabular}}
	\end{table*}
	
	\begin{table*}
		\caption{Comparison of the performance of node classification. The best result is bold. The second result is \underline{underlined}}
		\setlength{\tabcolsep}{1.2pt}
		\centering
		\begin{tabular}{c|c|ccc|ccc|ccc}
			\hline & \multirow{2}{*}{ Methods } & \multicolumn{3}{|c|}{ Recidivism } & \multicolumn{3}{|c|}{ Pokec-z } & \multicolumn{3}{|c}{ Pokec-n } \\
			\cmidrule(lr){3-5} \cmidrule(lr){6-8} \cmidrule(lr){9-11}
			& & $\mathrm{ACC}(\uparrow)$ & $\Delta_{D P}(\downarrow)$ & $\Delta_{E O}(\downarrow)$ & $\mathrm{ACC}(\uparrow)$ & $\Delta_{D P}(\downarrow)$ & $\Delta_{E O}(\downarrow)$ & $\mathrm{ACC}(\uparrow)$ & $\Delta_{D P}(\downarrow)$ & $\Delta_{E O}(\downarrow)$ \\
			\hline\multirow{6}{*}{ GCN } & Vanilla$\backslash \mathrm{S}$ & $83.89 \pm 0.76$ & $5.69 \pm 0.22$ & $\mathbf{3.42 \pm 0.55}$ & $69.74 \pm 0.22$ & $8.11 \pm 0.44$ & $6.41 \pm 0.47$ & $\underline{68.88 \pm 0.32}$ & $\underline{1.39 \pm 0.84}$ & $\underline{2.57 \pm 1.19}$ \\
			& RemoveR & $84.74 \pm 0.39$ & $6.08 \pm 0.21$ & $4.26 \pm 0.61$ & $\underline{69.77 \pm 0.31}$ & $8.00 \pm 0.48$ & $6.63 \pm 0.69$ & $68.65 \pm 0.24$ & $2.40 \pm 0.53$ & $3.45 \pm 1.10$ \\
			& KSMOTE & $83.93 \pm 0.81$ & $\underline{5.64 \pm 0.22}$ & $\underline{3.35 \pm 0.39}$ & $69.74 \pm 0.24$ & $7.99 \pm 0.48$ & $6.25 \pm 0.49$ & $68.69 \pm 0.29$ & $2.72 \pm 0.56$ & $3.60 \pm 0.64$ \\
			& FairRF & $83.96 \pm 0.84$ & $5.67 \pm 0.21$ & $4.40 \pm 0.36$ & $69.07 \pm 0.50$ & $\underline{7.57 \pm 0.60}$ & $6.77 \pm 0.81$ & $68.65 \pm 0.28$ & $2.79 \pm 0.61$ & $3.76 \pm 0.88$ \\
			& FairGKD$\backslash \mathrm{S}$ & $\underline{86.09 \pm 1.05}$ & $5.98 \pm 0.47$ & $4.47 \pm 0.72$ & $69.42 \pm 0.37$ & $7.72 \pm 0.90$ & $\underline{5.90 \pm 0.99}$ & $68.04 \pm 0.54$ & $2.46 \pm 0.86$ & $3.62 \pm 0.94$ \\
			
			&Fairwos & $\mathbf{86.56 \pm 2.74}$ & $\mathbf{5.06 \pm 1.42}$ &  $3.91 \pm 0.88$ & $\mathbf{70.60 \pm 0.34}$ & $\mathbf{5.03 \pm 1.70}$ & $\mathbf{4.96 \pm 2.02}$ & $\mathbf{70.44 \pm 0.34}$ & $\mathbf{1.25 \pm 0.78}$ & $\mathbf{1.83 \pm 1.03}$ \\
			
			\hline \multirow{6}{*}{ GIN } & Vanillal$\backslash \mathrm{S}$ & $72.49 \pm 5.19$ & $6.04 \pm 1.41$ & $4.90 \pm 1.64$ & $\underline{68.85 \pm 0.52}$ & $\underline{4.35 \pm 2.10}$ & $5.37 \pm 1.78$ & $68.41 \pm 0.58$ & $\underline{1.54 \pm 1.49}$ & $4.27 \pm 2.33$ \\
			& RemoveR & $72.81 \pm 5.46$ & $6.32 \pm 1.61$ & $4.22 \pm 1.88$ & $68.23 \pm 0.33$ & $4.96 \pm 0.65$ & $\underline{5.15 \pm 1.03}$ & $\underline{68.65 \pm 0.64}$ & $1.82 \pm 1.16$ & $\underline{3.48 \pm 1.75}$ \\
			& KSMOTE & $74.67 \pm 5.34$ & $6.21 \pm 1.52$ & $4.18 \pm 2.31$ & $68.71 \pm 0.47$ & $5.11 \pm 1.20$ & $5.66 \pm 1.69$ & $68.09 \pm 1.61$ & $2.22 \pm 2.21$ & $4.91 \pm 2.62$ \\
			& FairRF & $74.67 \pm 5.74$ & $5.97 \pm 1.48$ & $\underline{4.15 \pm 2.78}$ & $68.78 \pm 0.53$ & $5.31 \pm 1.14$ & $5.70 \pm 1.54$ & $68.23 \pm 1.57$ & $2.11 \pm 2.06$ & $5.08 \pm 2.31$ \\
			& FairGKD$\backslash \mathrm{S}$ & $\mathbf{79.98 \pm 2.45}$ & $\underline{5.89 \pm 1.56}$ & $4.35 \pm 0.96$ & $\mathbf{68.87 \pm 0.35}$ & $4.78 \pm 1.03$ & $5.44 \pm 1.38$ & $68.48 \pm 0.63$ & $2.10 \pm 0.81$ & $4.57 \pm 1.32$ \\

			& Fairwos  & $\underline{79.80 \pm 3.43}$ & $\mathbf{4.36 \pm 0.94}$ & $\mathbf{3.55 \pm 1.15}$ & $68.35 \pm 1.17$ & $\mathbf{3.49 \pm 2.14}$ & $\mathbf{2.83 \pm 1.63}$ & $\mathbf{69.27 \pm 0.76}$ & $\mathbf{1.49 \pm 1.47}$ & $\mathbf{2.78 \pm 1.78}$ \\

			\hline & \multirow{2}{*}{ Methods } & \multicolumn{3}{|c|}{$N B A$} & \multicolumn{3}{|c|}{ Credit } & \multicolumn{3}{|c}{ Occupation } \\
			\cmidrule(lr){3-5} \cmidrule(lr){6-8} \cmidrule(lr){9-11}
			& & $\mathrm{ACC}(\uparrow)$ & $\Delta_{D P}(\downarrow)$ & $\Delta_{E O}(\downarrow)$ & $\mathrm{ACC}(\uparrow)$ & $\Delta_{D P}(\downarrow)$ & $\Delta_{E O}(\downarrow)$ & $\mathrm{ACC}(\uparrow)$ & $\Delta_{D P}(\downarrow)$ & $\Delta_{E O}(\downarrow)$ \\
			\hline\multirow{6}{*}{ GCN } & Vanilla$\backslash \mathrm{S}$ & $\underline{66.38 \pm 1.64}$ & $28.34 \pm 3.86$ & $23.70 \pm 4.88$ & $\mathbf{73.77 \pm 1.57}$ & $11.63 \pm 4.25$ & $9.58 \pm 3.83$ & $\mathbf{81.99 \pm 0.34}$ & $28.56 \pm 0.56$ & $17.10 \pm 1.53$ \\
			& RemoveR & $65.92 \pm 1.80$ & $24.82 \pm 5.83$ & $19.04 \pm 7.42$ & $\underline{73.71 \pm 0.52}$ & $12.88 \pm 2.97$ & $10.32 \pm 2.14$ & $\underline{81.82 \pm 0.16}$ & $28.71 \pm 0.48$ & $16.44 \pm 0.31$ \\
			& KSMOTE & $60.80 \pm 0.57$ & $23.53 \pm 2.37$ & $17.40 \pm 6.87$ & $73.54 \pm 0.08$ & $11.89 \pm 0.31$ & $9.61 \pm 0.33$ & $79.88 \pm 0.12$ & $27.65 \pm 0.32$ & $17.40 \pm 0.53$ \\
			& FairRF & $61.13 \pm 0.66$ & $23.44 \pm 3.11$ & $17.28 \pm 7.50$ & $73.57 \pm 0.08$ & $ \underline{11.73 \pm 0.15}$ & $9.47 \pm 0.18$ & $79.74 \pm 0.10$ & $\underline{27.22 \pm 0.28}$ & $17.12 \pm 0.67$ \\
			& FairGKD$\backslash \mathrm{S}$ & $62.02 \pm 2.13$ & $\underline{16.52 \pm 6.99}$ & $\underline{12.32 \pm 7.56}$ & $73.33 \pm 0.19$ & $11.98 \pm 0.48$ & $9.65 \pm 0.42$ & $81.49 \pm 0.36$ & $27.74 \pm 1.09$ & $\underline{14.20 \pm 1.60}$ \\
			
			& Fairwos & $\mathbf{68.22 \pm 1.65}$ & $\mathbf{10.16 \pm 5.20}$ & $\mathbf{7.16 \pm 4.83}$ & $73.54 \pm 3.46$ & $\mathbf{9.22 \pm 5.44}$ & $\mathbf{7.55 \pm 4.85}$ & $81.76 \pm 0.61$ & $\mathbf{25.16 \pm 1.31}$ & $\mathbf{13.34 \pm 1.95}$ \\
			
			\hline \multirow{6}{*}{ GIN } & Vanilla$\backslash \mathrm{S}$ & $59.15 \pm 2.13$ & $11.66 \pm 6.63$ & $13.78 \pm 7.90$ & $74.52 \pm 0.83$ & $11.37 \pm 1.97$ & $8.93 \pm 1.96$ & $\underline{80.82 \pm 0.43}$ & $25.70 \pm 1.40$ & $14.90 \pm 2.43$ \\
			& RemoveR & $\underline{61.55 \pm 1.56}$ & $13.49 \pm 6.90$ & $12.29 \pm 7.89$ & $\underline{74.72 \pm 0.75}$ & $12.54 \pm 1.52$ & $11.54 \pm 1.50$ & $\mathbf{81.10 \pm 0.48}$ & $27.77 \pm 1.24$ & $13.16 \pm 2.17$ \\
			& KSMOTE & $58.69 \pm 3.20$ & $15.07 \pm 11.61$ & $14.71 \pm 11.66$ & $73.55 \pm 2.20$ & $\underline{10.71 \pm 3.06}$ & $\underline{8.00 \pm 3.12}$ & $77.54 \pm 3.63$ & $\underline{22.35 \pm 5.71}$ & $12.53 \pm 4.84$ \\
			& FairRF & $57.56 \pm 4.34$ & $11.38 \pm 10.15$ & $11.4 \pm 10.56$ & $73.58 \pm 1.52$ & $10.89 \pm 2.15$ & $8.28 \pm 2.33$ & $79.50 \pm 1.04$ & $24.55 \pm 1.72$ & $\mathbf{11.29 \pm 3.29}$ \\
			& FairGKD$\backslash S$ & $61.13 \pm 2.71$ & $\underline{9.53 \pm 3.07}$ & $\underline{7.61 \pm 5.60}$ & $73.31 \pm 1.73$ & $10.80 \pm 6.01$ & $8.54 \pm 6.86$ & $80.41 \pm 0.50$ & $25.64 \pm 1.30$ & $13.52 \pm 2.16$ \\

			& Fairwos & $\mathbf{62.77 \pm 2.22}$ & $\mathbf{8.69 \pm 7.28}$ & $\mathbf{7.01 \pm 4.90}$ & $\mathbf{75.73 \pm 2.76}$ & $\mathbf{9.42 \pm 3.93}$ & $\mathbf{7.74 \pm 3.39}$ & $80.61 \pm 0.80$ & $\mathbf{19.98 \pm 2.97}$ & $\underline{12.08 \pm 1.67}$ \\

			\hline

		\end{tabular}
		
	\end{table*}

	\section{Experiments}
	
	In this section, we conduct experiments to evaluate the proposed Fairwos's effectiveness in fairness and classification performance when sensitive attributes are unavailable. In particular, we aim to answer the following research questions: 
	
		\begin{itemize}
		\item (\textbf{RQ1}): Can the proposed Fairwos achieve fairness without sensitive attributes while maintaining high utility?
		\item (\textbf{RQ2}): How do the encoder, fairness promotion, and weight updating modules affect Fairwos?
		\item (\textbf{RQ3}): How does the dimension of the encoder affect the model performance?
		\item (\textbf{RQ4}): How would different choices of hyper-parameters influence the performance of Fairwos?
       \item (\textbf{RQ5}): What do pseudo-sensitive attributes represent when sensitive attributes are hidden during training?
       \item  (\textbf{RQ6}): How does Fairwos' runtime performance compare to baseline methods and its variants?

	\end{itemize}

	\subsection{Experimental Settings}
	
	\subsubsection{Datasets} 
	
	We conduct the experiments on six publicly available benchmark datasets, including Bail \cite{ref32}, Credit \cite{ref33}, NBA \cite{ref12}, Pokec-z \cite{ref34}, Pokec-n \cite{ref34}, and Occupation \cite{ref35}. The statistics of the above datasets are given in Table I and the details of the datasets are as follows:

	\textbf{Bail:} The bail dataset has 18876 nodes representing defendants released on bail during 1990-2009. The nodes are connected based on the similarity calculated based on past criminal records and demographics. The task is to classify the defendants into bail or not. Race is considered a sensitive attribute. 
	
	\textbf{Credit:} The credit dataset consists of 30000 nodes and the nodes with similar spending and payment patterns are connected. The task is to predict whether or not the user is paying using a credit card by default. Age is treated as a sensitive attribute. 
	
	\textbf{NBA:} The NBA dataset contains around 403 players with information from the 2016-2017 season. The link data is crawled from Twitter. The nationality is set as the sensitive attribute. The classification task is to predict whether the player's salary is above average. 
	
	\textbf{Pokec:} Pokec dataset consists of Pokec-z and Pokec-n. They were sampled from Slovakia's most popular social network in 2012, which belongs to two major regions of the corresponding provinces. Nodes denote users, and an edge represents the friendship between nodes. The region is regarded as a sensitive attribute. The task is to predict the users' working field. 
	
	\textbf{Occupation:} The occupation dataset originated from Twitter, where the node denotes the user and the edges represent the friendships. The task is to predict whether a user focuses on computer science or psychology. Gender is regarded as a sensitive attribute.

	\subsubsection{Metrics}

	We evaluate the model performance from two perspectives: classification performance and fairness. We adopt ACC to measure the performance on node classification tasks for classification performance. The ACC close to 1 indicates better classification performance. For fairness performance, we follow existing work on fair GNN to use two widely applied evaluation metrics, i.e., equal opportunity ($\Delta_{EO}$) \cite{ref29} and demographic parity ($\Delta_{SP}$) \cite{ref36}. Specifically, the $\Delta_{SP}$ is computed as: 
	\begin{equation}
		\Delta_{S P}=|P(\hat{y}=1 \mid s=0)-P(\hat{y}=1 \mid s=1)|
	\end{equation}
	and $\Delta_{EO}$ is computed as 
	\begin{equation}
		\Delta_{E O}=|P(\hat{y}=1 \mid y=1, s=0)-P(\hat{y}=1 \mid y=1, s=1)|
	\end{equation}
	where the probability $P()$ is estimated on the test dataset. The model is better when $\Delta_{E O}$ and $\Delta_{SP}$ are smaller.
	
	\subsubsection{Baselines} 
	
	To evaluate the effectiveness of Fairwos, we compare Fairwos with the following fair learning without sensitive attributes.

	\textbf{Vanilla$\backslash \mathrm{S}$}: It is the backbone GNN, which is trained without sensitive attributes. 
	
	\textbf{KSMOTE}\cite{ref37}: This obtains the pseudo-groups by clustering and enforces fairness by regularizing model prediction based on these pseudo-groups.
	
	\textbf{RemoveR}: It directly pre-processes the graph data by removing all the candidate-related attributes.
	
	\textbf{FairRF}\cite{ref23}: This minimizes the correlation with the prediction with related sensitive attributes.
	
	\textbf{FairGKD$\backslash \mathrm{S}$}\cite{ref27}: It adopts partial knowledge distillation for learning fair graph representation without sensitive attributes.

	Since KSMOTE, and FairRF are not designed for GNNs, we directly use the code provided by \cite{ref23, ref37} on our backbone GNN.

	\subsubsection{Implementation Details} 
	
	The experiments were conducted on a machine with Intel(R) Xeon(R) Platinum 8352V CPU @ 2.10GHz and NVIDIA GeForce RTX 4090 with 24GB memory and CUDA 11.8. The operating system is Ubuntu 20.04 with 120GB memory. Each dataset is randomly split into 50\%/25\%/25\% training/ validation/ testing data. The ADM optimizer with a learning rate of 0.001 accomplishes the optimization. For the backbone model, the layer and hidden unit number are set as 1 and 16, respectively. For hyperparameter selection, we vary $\alpha$ as $\{0.01, 0.05, 1, 2, 5\}$ and $K$ as $\{1, 2, 5, 10, 20\}$ and the best model is saved based on the performance of the validation dataset. The hidden unit layer is determined through a grid search with values from 4 to 16. We pre-train a GNN model with high utility and then fine-turning it to promote fairness. The epoch of the first process is set as 1000 and the second process as 15. We use early stop operation to preserve competitive utility performance.

	\subsection {Performance Comparison}
	
	To answer \textbf{RQ1}, we experiment on six datasets with a comparison to the baselines. Each experiment is conducted 10 times. The performance of prediction and fairness with standard deviation is shown in Table II. The best results are highlighted in bold and the second result is underlined. From Table II, we have the following observations. 
	
	\begin{itemize}
		\item From the perspective of model utility, our proposed Fairwos provides competitive performance. Compared with the Vanilla model, the Fairwos improves fairness without causing a significant performance drop. Besides, the ACC is higher than some baselines. The reason is the encoder module learns the crucial information while reducing the dimensions of attributes. This process removes noise and redundant information, which improves the model's generalization. With fewer dimensions, the risk of overfitting decreases, allowing the model to achieve higher accuracy on new data.
		\item From the perspective of fairness, our proposed framework outperforms most of the baseline methods that learn fair representation without sensitive attributes. On the one hand, these baselines are designed for independently and identically distributed data rather than graph data. On the other hand, we focus on eliminating the root cause of unfairness and getting rid of the correlation-induced abnormal behaviors. Furthermore, Fairwos performs better than state-of-the-art (SOTA) methods that learn fair graph representation without sensitive attributes.  
		\item From the perspective of balancing the model utility and fairness, Fairwos achieves competitive performance on the prediction and superior fairness promotion in all cases compared with other baselines. In some cases, both the utility and fairness of Fairwos outperform the SOTA method. We claim that Fairwos achieves better performance on balance utility and fairness.
	\end{itemize}

	These observations demonstrate that Fairwos achieves superior performance in balancing utility and fairness.

\begin{figure}[t]
		\centering
		\captionsetup{justification=centering}
		\includegraphics[width=0.5\textwidth]{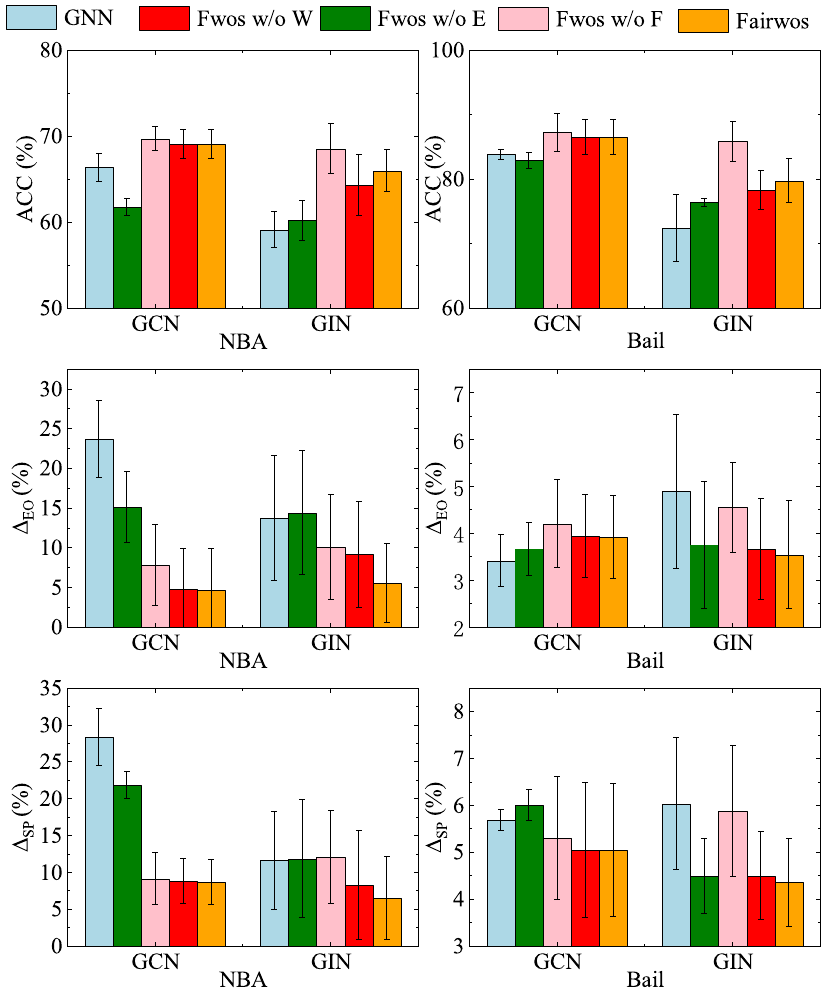} % 调整width参数以控制图片宽度
		\caption{Comparisons between Fairwos and its variants on NBA and Bail dataset}
		\label{fig:example}
	\end{figure}
	
	\begin{figure}[t]
		\centering
		\captionsetup{justification=centering}
		\includegraphics[width=0.5\textwidth]{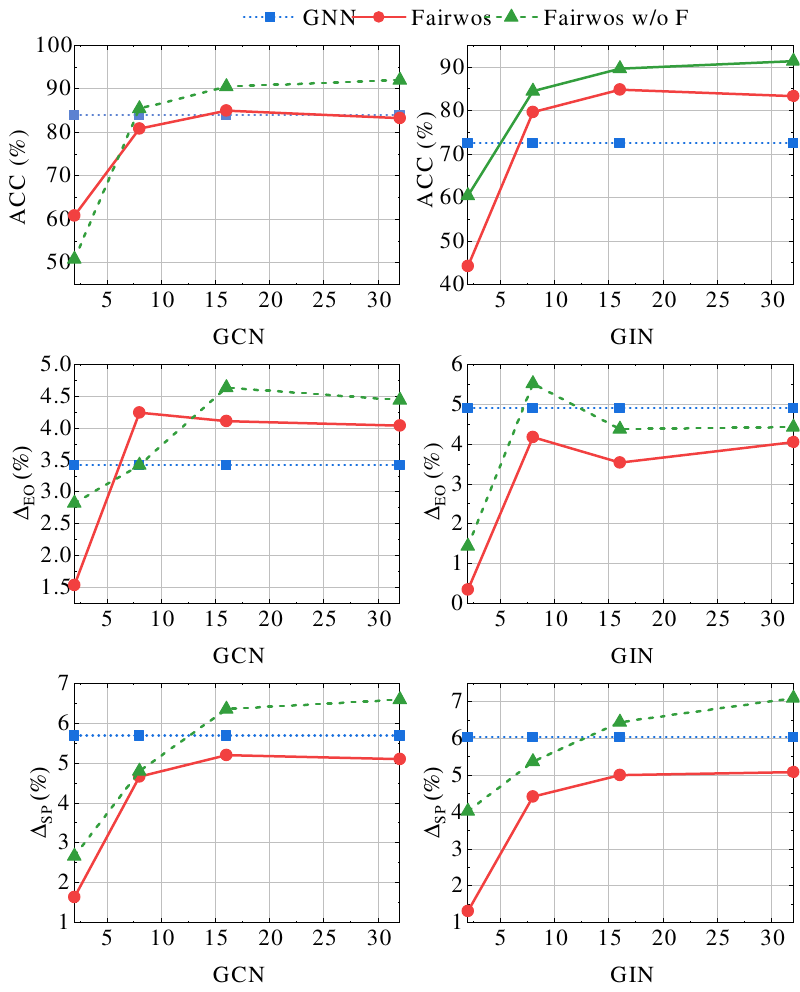} % 调整width参数以控制图片宽度
		\caption{Impacts of the dimension of the encoder to Fairwos on GCN}
		\label{fig:example}
	\end{figure}

    	\begin{figure*}[t]
		\centering
		\begin{subfigure}{0.36\textwidth}
			\includegraphics[width=\linewidth]{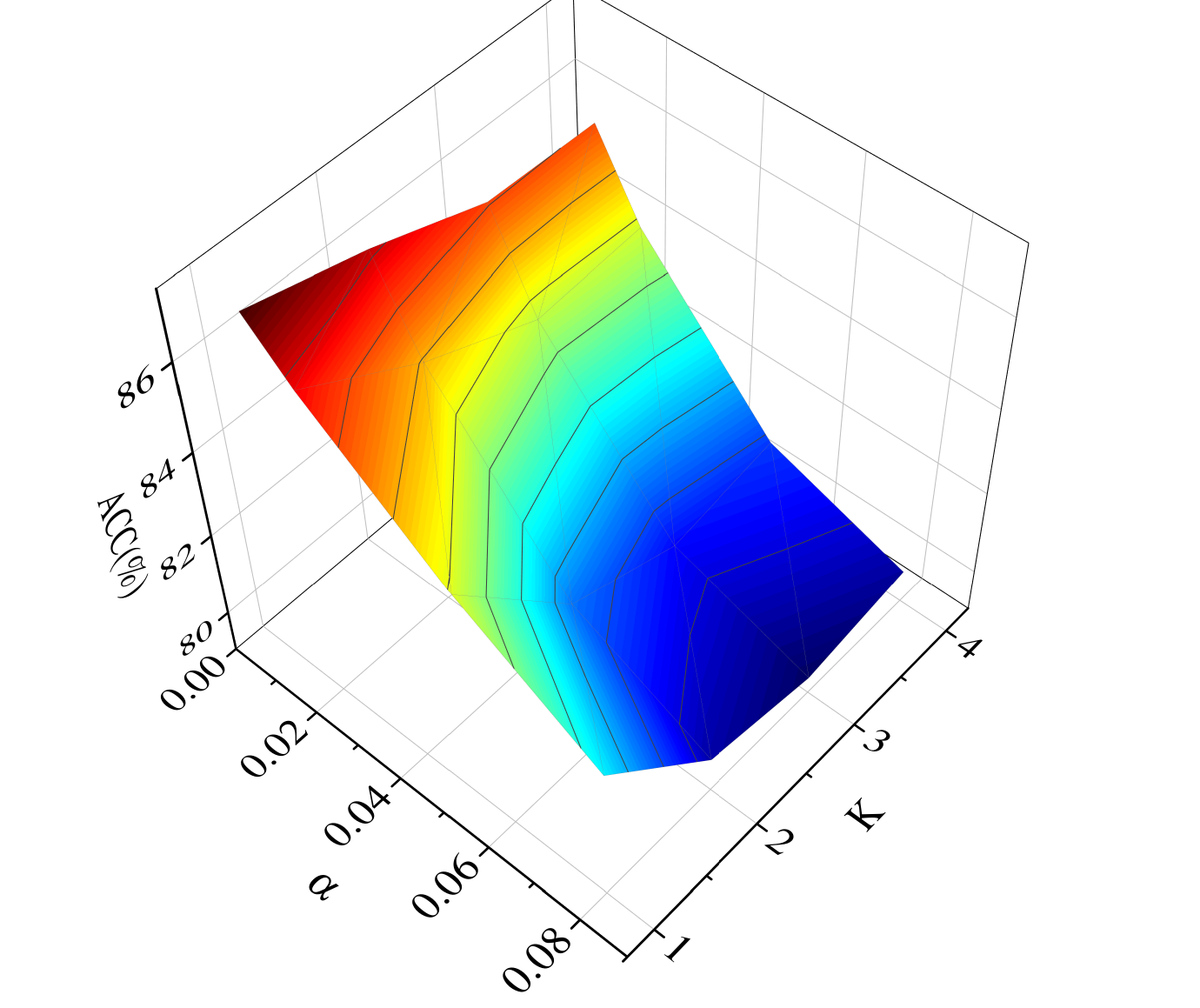}
			\caption{ACC}
		\end{subfigure}
		\hspace{-100pt}
		\hfill
		\begin{subfigure}{0.36\textwidth}
			\includegraphics[width=\linewidth]{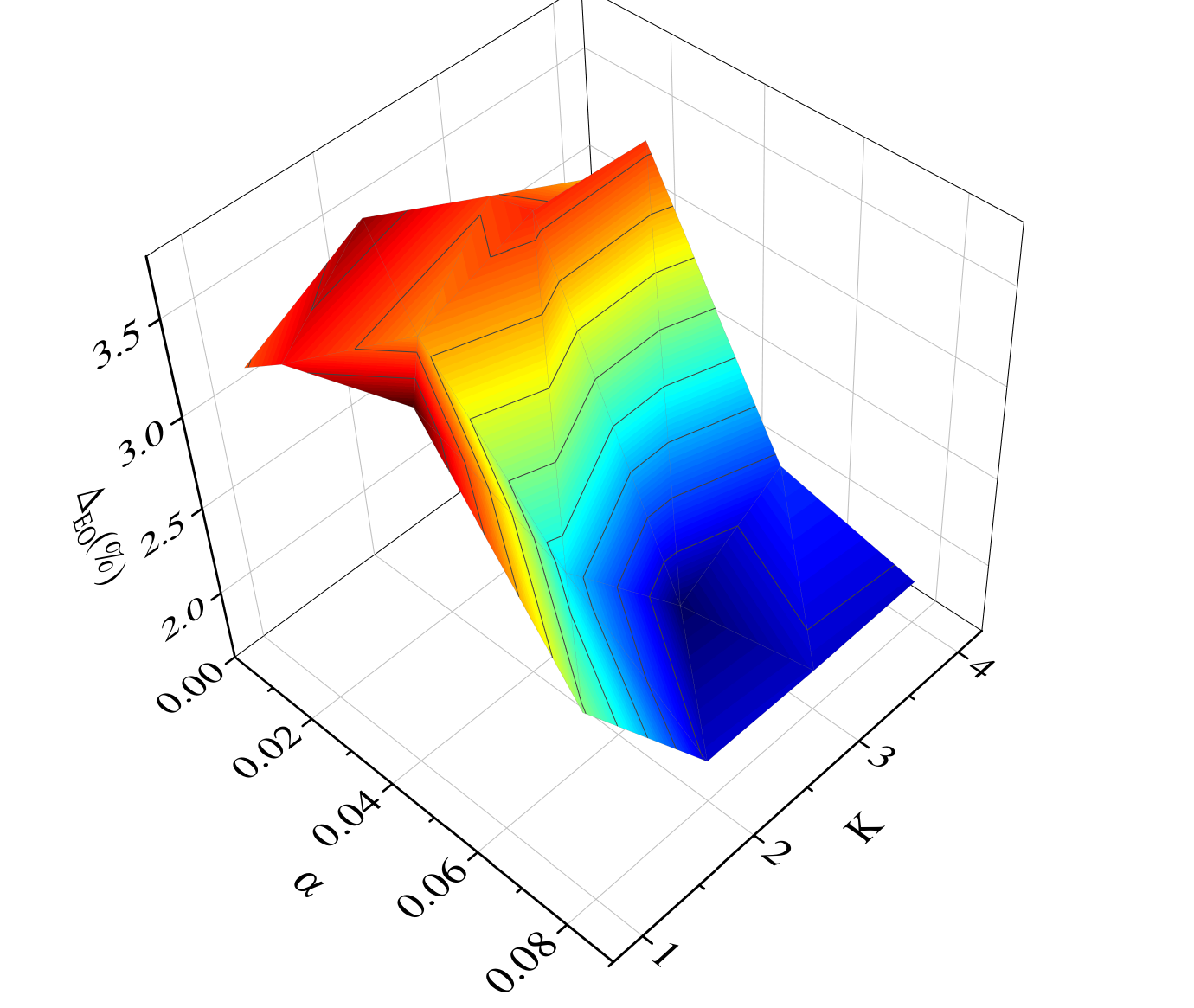}
			\caption{$\Delta_{EO}$}
		\end{subfigure}
		\hspace{-100pt}
		\hfill
		\begin{subfigure}{0.36\textwidth}
			\includegraphics[width=\linewidth]{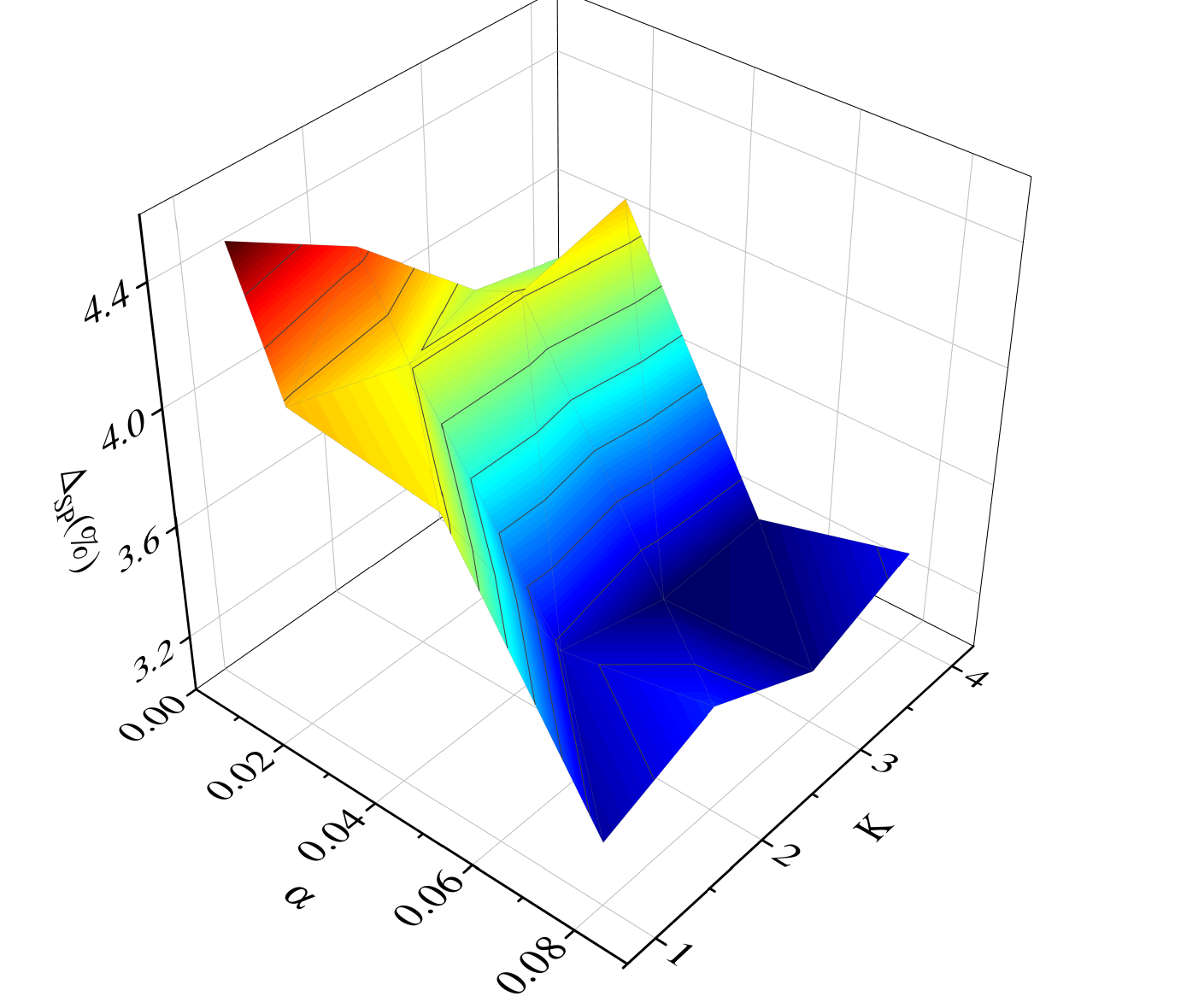}
			\caption{$\Delta_{SP}$}
		\end{subfigure}
		\caption{Hyper-parameter study on Bail dataset}
	\end{figure*}

	\subsection{Ablation Studies}

	To answer \textbf{RQ2}, we conduct an ablation study to fully understand the effect of each component of Fairwos on promoting fairness. Specifically, we denote \textbf{Fwos w/o E} as removing the encoder module, \textbf{Fwos w/o F} as removing the module of fairness promotion, and \textbf{Fwos w/o W} as removing the weight updating module. Similarly, for each variant, we conducted the experiment 10 times on bail and NBA, respectively. The average performance of fairness and utility with standard deviation is shown in Fig. 4. We adopt two widely applied GNNs as the backbone model of Fairwos, i.e., GCN and GIN.  From Fig. 4, we observe:
	
	\begin{itemize}
		\item The ACC of \textit{Fwos w/o E} is lower than GNN and other variants such as \textit{Fwos w/o F}, \textit{Fwos w/o W} and \textit{Fairwos} are higher than GNN. This demonstrates the encoder's ability to enhance utility. On the one hand, the encoder reduces the noise and redundant information and preserves the information that is more useful for the task. On the other hand, this way prevents overfitting and helps GNN to learn more general representation.
		
		\item The $\Delta_{SP}$ and $\Delta_{EO}$ of \textit{Fwos w/o E}, \textit{Fwos w/o F}, and \textit{Fwos w/o W} are smaller than those of GNNs in Fig. 4. This indicates that the encoder, fairness promotion, and weight updating modules all contribute to improving fairness. By using the encoder, fairness promotion is applied to pseudo-sensitive attributes instead of all non-sensitive attributes individually. This shows that promoting fairness through pseudo-sensitive attributes is more effective than using all non-sensitive attributes. The fairness promotion module further reduces the influence of sensitive attributes, while the weight learning module minimizes over-reliance on specific attributes. Therefore, \textit{Fairwos} achieves better results compared to other alternatives.
		
	\end{itemize}

	\subsection{Hyper-Parameter Sensitivity Analysis}
	
	To answer \textbf{RQ3}, we explore the sensitivity of the encoder's dimension. We alter the value of the dimension among \{2, 8, 16, 32\} and conduct experiments on different GNN backbones: GCN and GIN. The results are presented in Fig. 5, respectively. From these figures, we can find that the decrease in dimension leads to a decrease in terms of accuracy and bias. However, the presence of the encoder makes the accuracy of \textit{Fairwos w/o F} still higher than the backbone model until dimension 8. As the dimension continues to decrease, the accuracy of the model is much lower than the backbone model and the bias of the model decreases substantially. This indicates that the encoder can improve the utility of the model. However, as the dimension decreases, too much information is compressed which is no longer sufficient to support accurate prediction. 
	
	To answer \textbf{RQ4}, we analyze the sensitivity of Fairwos on two important hyperparameters, i.e., $\alpha$ and $K$. $\alpha$ controls the importance of the fairness regularization term and $K$ controls the number of graph counterfactuals. Specifically, we vary $\alpha$ as $\{0.01, 0.02, 0.04, 0.08\}$ and $K$ as $\{1, 2, 3, 4\}$ and ensure other parameters remain unchanged. 	The experiments are conducted on the bail dataset, and the results are shown in Fig. 6. From the figure, we can observe: that there exists a trade-off between utility and fairness. Increasing $\alpha$ and $K$ both improve the fairness. However, larger $\alpha$ and $K$ cause severe drops in the utility when they are larger than some threshold. Besides, the fairness can't be improved when  $\alpha$ or $K$ are smaller than a certain threshold.

    \subsection{Visualization of Pseudo-sensitive Attributes}

    To address \textbf{RQ5}, we visualize the pseudo-sensitive attributes on the NBA and Occupation dataset to better understand what they represent. Subsequently, we use t-SNE \cite{ref38} to map the pseudo-sensitive attribute into 2-dimensional space for visualization. We select samples only from the test set to enhance clarity, which aligns with our assumption that sensitive attributes are accessible only during testing. The results in Fig. 7(a) and Fig. 7(b) show some separation between clusters based on the pseudo-sensitive attributes. This suggests that the pseudo-sensitive attributes capture certain aspects of the sensitive attributes. Since pseudo-sensitive attributes replace non-sensitive attributes as inputs, any influence of sensitive attributes on the model's predictions occurs indirectly through these pseudo-sensitive attributes. Hence, we aim to minimize this indirect influence by reducing the correlation between pseudo-sensitive attributes and model predictions.

  \begin{figure}[t]
    \centering
    % 第一张图
    \begin{subfigure}[b]{0.24\textwidth}
        \centering
        \includegraphics[width=\textwidth]{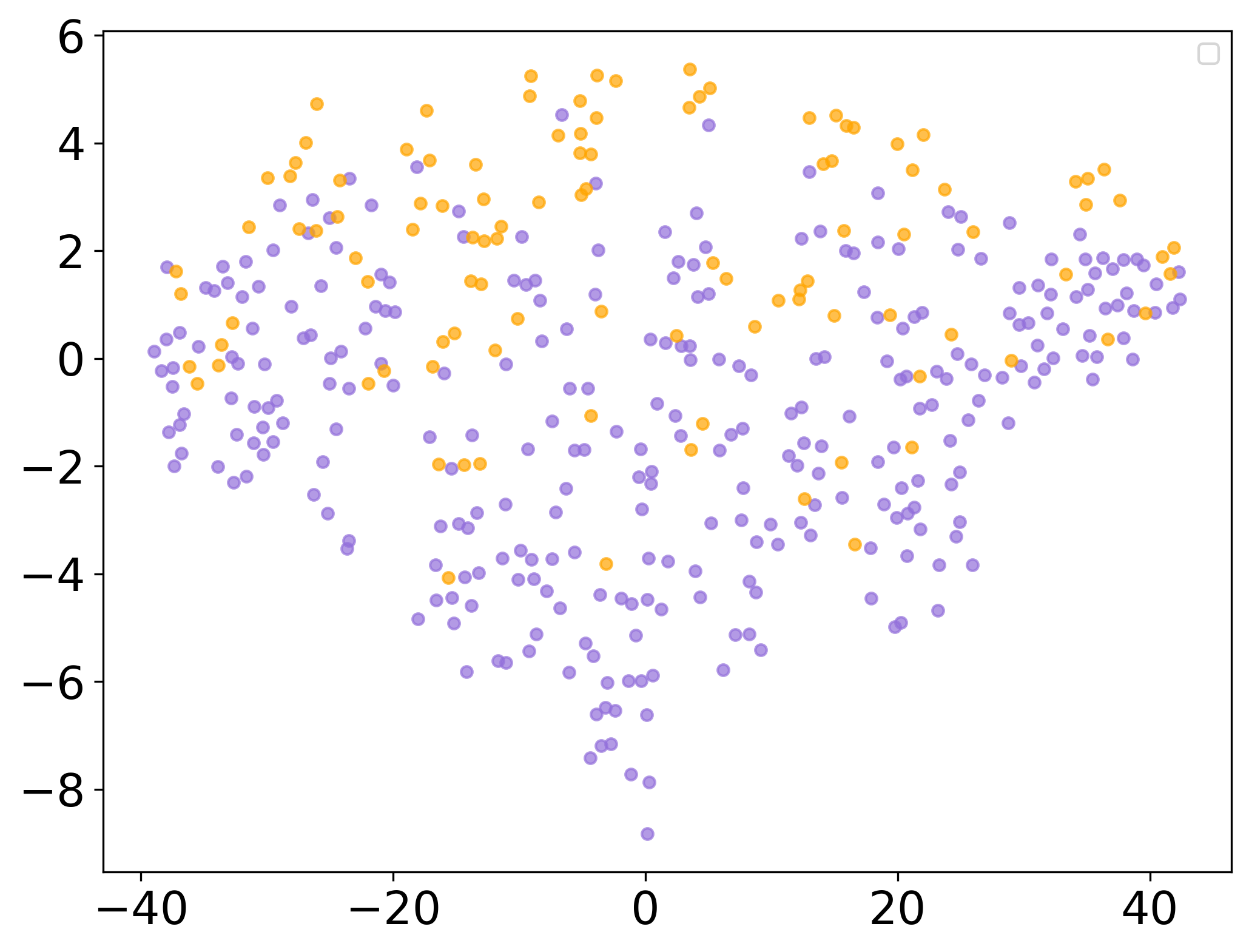}
        \caption{ NBA dataset}
        \label{fig:nba_plot}
    \end{subfigure}
    \hfill
    % 第二张图
    \begin{subfigure}[b]{0.24\textwidth}
        \centering
        \includegraphics[width=\textwidth]{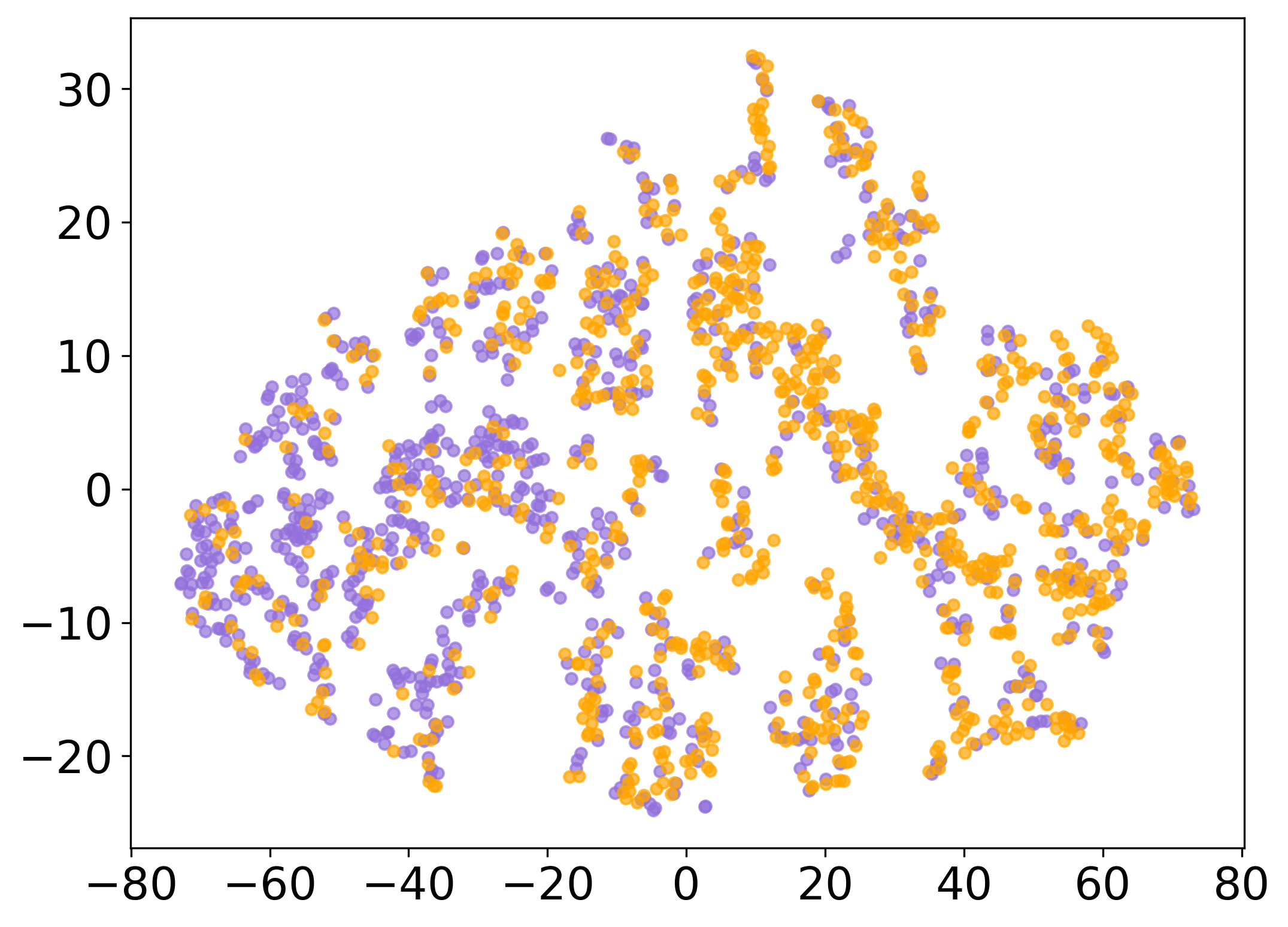}
        \caption{Occupation dataset}
        \label{fig:occupation_plot}
    \end{subfigure}
    
    \caption{Visualizations of pseudo-Sensitive attributes for NBA and Occupation datasets. Dot colors represent two distinct demographic groups based on sensitive attributes.
 }
    \label{fig:comparison_plots}
\end{figure}

\subsection{Runtime Analysis and Efficiency Evaluation}
    To address \textbf{RQ6}, we conduct a runtime analysis to evaluate the efficiency of Fairwos and its different variants. Specifically, these variants are explained previously in the ablation study: \textbf{Fwos w/o W}, \textbf{Fwos w/o E}, and \textbf{Fwos w/o F}. We also compare Fairwos with other baselines such as GNN, RemoveR, KSMOTE, FairRF, and FairGKD$\backslash \mathrm{S}$. The experiment is repeated 10 times on the NBA dataset, and the mean running time along with the standard deviation is presented in Fig. 8. From Fig. 8, we observe:

    \begin{itemize}
     \item Among the baselines, RemoveR has the shortest runtime for both GCN and GIN, as it removes some attributes and trains with fewer feature dimensions. Besides, KSMOTE and FairRF also show similar or slightly lower runtimes compared to the full Fairwos model. Although they achieve comparable computational efficiency, Fairwos provides additional advantages in enhancing fairness. FairGKD$\backslash S$ has a significantly longer training time compared to Fairwos. This is because FairGKD$\backslash S$ requires training two teacher models, which are then used to guide the student model through knowledge distillation. This multi-stage process adds significant computational overhead, whereas Fairwos uses a more efficient approach, resulting in a shorter runtime.

    \item  Among the Fairwos variants, the training time of Fairwos is longer than both \textit{Fwos w/o F} and \textit{Fwos w/o W}, which indicates that the weight updating module and fairness module both increase the training time. The reason is that the fairness module introduces additional computations to enhance fairness, and the weight updating module requires iterative adjustments during training. However, Fairwos has a significantly shorter runtime compared to \textit{Fwos w/o E}. This is because the encoder module reduces the feature dimensions, which avoids applying fairness promotion on each non-sensitive attribute and reduces the training time.

\end{itemize}

    \begin{figure}[t]
    \centering
    \includegraphics[width=0.4\textwidth]{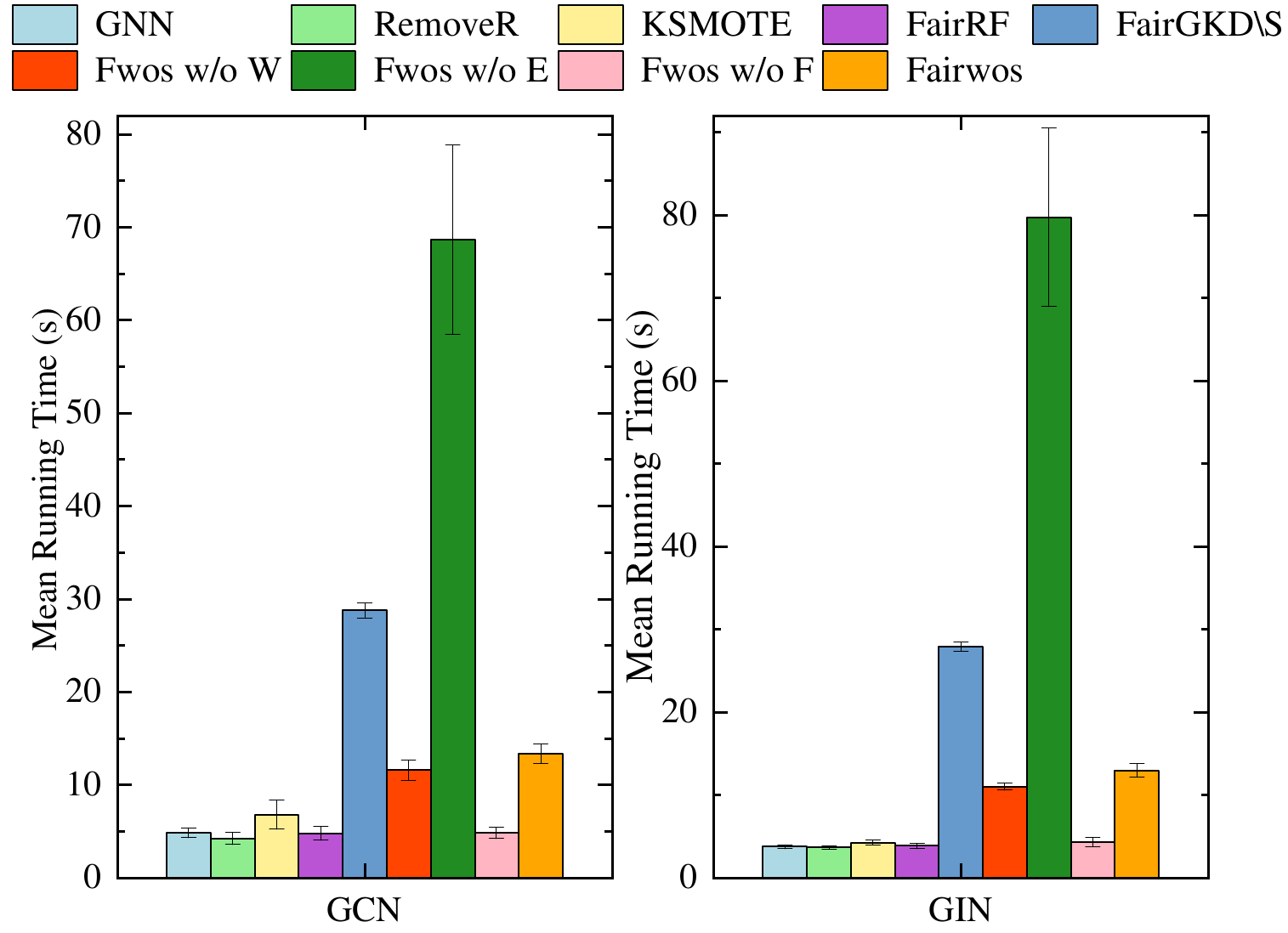}  % 设置图片路径和宽度
    \caption{Runtime comparison of Fairwos, Fairwos's variants, and baselines}  % 图片说明
    \label{fig:time_compare}  % 图片标签，方便引用
\end{figure}
    
	\section{Related Work}
	
    \subsection{Graph Neural Networks} 
	
	Graph neural networks (GNNs)  are mainly divided into two categories: spectral GNNs and spatial-based GNNs \cite{ref39}. Spectral GNNs are based on spectral graph theory. They inherit the insights from the Convolution Neural Network (CNN) \cite{ref40}, which designs graph filters to extract task-related information from the input graph. Many works follow spectral GNNs for further improvement and extensions. In these works, a Graph Convolutional Network (GCN) \cite{ref30} is the most popular work that leverages graph convolution operation to capture the information. Different from spectral GNNs, spatial GNNs leverage message-passing mechanisms to aggregate the node’s representation. For example, the Graph Attention Network (GAT) \cite{ref41} introduces an attention mechanism, aligning importance weight for the neighbor’s attributes during the aggregation. GraphSAGE \cite{ref42} utilizes a neighbor sampling method to address the scalability of GCN. Graph Isomorphism Network (GIN) \cite{ref31} can capture the isomorphism of graphs. It enhances the expressive power of GNN by introducing learnable aggregation functions and update functions. Although GNNs show extensive utility and efficiency, they may inherit bias from the graph data, and the aggregation mechanism emphasizes the bias. There, it is essential to develop fair GNNs, especially for high-stake decision-making scenarios.

	\subsection{Fair graph neural networks}

	Fairness has received a lot of attention in machine learning, recently. Most of the existing methods are based on statistical fairness notation such as group fairness \cite{ref12} and individual fairness \cite{ref43}. They ensure that GNNs make a fair prediction about the protected group or individuals. To measure fairness, statistical parity \cite{ref36}, and equal opportunity \cite{ref29} are the widely applied metrics. FairGNN \cite{ref12} as an in-processing technique, refines the representation using adversarial learning to prevent an adversary from predicting the sensitive attribute accurately. Edits \cite{ref3.5} is a pre-processing technique, which eliminates the bias raised from two aspects including node attributes and network topology. Despite these methods achieving good performance, they focus on correlation and fail to mitigate the bias when encountering statistical anomalies. To address this limitation, counterfactual fairness leverages the causal theory to eliminate the root bias. There are some initial works to develop fairness-aware GNN with counterfactual fairness. NIFTY \cite{ref11} first perturbs the edge and the sensitive attribute to generate counterfactuals and maximizes the agreement between the representation learned from the original and counterfactuals. GEAR \cite{ref20} generates counterfactuals using GraphVAE \cite{ref44} and minimizes the disparity between the representations learned from the original data and the counterfactual. CAF \cite{ref15} finds counterfactuals from existing representation space with guidance of labels and sensitive attributes. Based on these counterfactuals, it designs a novel constraint to learn fair representation. 
	
	Our work is inherently different from existing works: (i) we focus on promoting fairness without sensitive attributes via graph counterfactuals; (ii) pseudo-sensitive attributes are adopted to generate graph counterfactuals as supervision for regularization.

	\subsection{Fairness without demographics}
	
	Due to the law and privacy issues, it is necessary to explore fairness-aware GNNs without sensitive attributes. This problem can be addressed in two directions. One direction is regularizing the model prediction using the pseudo-sensitive attributes. For example, Yan et al. \cite{ref37} generate the pseudo-sensitive attributes using a clustering algorithm. Zhao et al. \cite{ref23} leverage the non-sensitive attribute as the pseudo-sensitive attribute. Another direction focuses on the Min-Max fairness, which improves the subgroup with worst-case performance. For instance, Lahoti et al. \cite{ref22} achieve this fairness using adversarial reweighted learning to maximize the loss with the notation of computationally-identifiability. Chai et al. \cite{ref26} also follow a similar idea of Min-Max fairness and learn a generally fair representation of fully unsupervised training data. Chai et al. adopt knowledge distillation to learn fairness representation without sensitive attributes and demonstrate their method can be seen as an error-based reweighing. Zhu et al. \cite{ref27} leverage the knowledge distillation into graph data and propose to train the teacher model using partial data, which makes the student model learn more fair representation.

	Different from existing works, our paper focuses on fairness without sensitive attributes on graph data and eliminates the root of bias based on causal theory.

	\section{Conclusion and future work}
	
	In this paper, we study the problem of learning fair node representation without sensitive attributes. Due to privacy and legal issues, sensitive attributes are unavailable for training. Existing studies have demonstrated that bias also exists in this situation. Firstly, we analyze the relation between the sensitive attributes and the prediction. We find the sensitive bias is caused by the non-sensitive information which consists of the graph structure and the non-sensitive attributes. We learn the representation of the non-sensitive information with low dimensions and adopt all the processed attributes as pseudo-sensitive attributes. The aim is to ensure that the prediction does not over-depend on any of the pseudo-sensitive attributes. We generate the graph counterfactual by searching from the real data set with the guidance of pseudo-sensitive attributes and the prediction. Then, fairness is ensured by minimizing the disparity between the representations learned from the original graph and the graph counterfactual. To further balance fairness and utility, we assign an importance weight to each pseudo-sensitive attribute and introduce a weight updating mechanism to adjust these weights dynamically. This is because some pseudo-sensitive attributes are highly influential for the prediction and contribute to the bias, differently. Experiments on real-world datasets demonstrate the efficiency of our method in balancing utility and fairness while sensitive attributes are unavailable.

\bibliographystyle{unsrt}

\vspace{12pt}

\end{document}